\documentclass[letterpaper,10pt]{article}

\usepackage{probml}

\ShortHeadings{When Individually Calibrated Models Become Collectively Miscalibrated}

\usepackage{algorithm}
\usepackage{algorithmic}
\usepackage{multirow}
\usepackage[table,xcdraw]{xcolor}
\usepackage{colortbl}
\usepackage{wrapfig}
\usepackage{caption}

\definecolor{bestbg}{HTML}{D5F5E3}    %
\definecolor{ourbg}{HTML}{EBF5FB}     %
\definecolor{rowalt}{HTML}{F8F9FA}    %
\newcommand{\best}[1]{\cellcolor{bestbg}\textbf{#1}}
\newcommand{\bestm}[1]{\cellcolor{bestbg}$\mathbf{#1}$}  %
\newcommand{\ours}{\rowcolor{ourbg}}

\usepackage{bm,amssymb,amsmath}

\makeatletter
\def\th@plain{%
  \thm@notefont{}%
  \itshape %
}
\def\th@definition{%
  \thm@notefont{}%
  \normalfont %
}
\makeatother

\def\1{\bm{1}}

\DeclareMathAlphabet{\mathsfit}{\encodingdefault}{\sfdefault}{m}{sl}
\SetMathAlphabet{\mathsfit}{bold}{\encodingdefault}{\sfdefault}{bx}{n}

\newcommand{\E}{\mathbb{E}}

\newcommand{\Var}{\mathrm{Var}}

\newcommand{\Cov}{\mathrm{Cov}}

\declaretheorem[name=Conjecture, sibling=theorem]{conjecture}

\begin{document}

\title{
  When Individually Calibrated Models Become Collectively Miscalibrated
}

\author[1]{Zhaohui Wang\thanks{Correspondence: \texttt{zwang000@usc.edu}}}
\affil[1]{USC Viterbi School of Engineering, University of Southern California}

\maketitle

\begin{abstract}
Probabilistic prediction systems often aggregate probability estimates from multiple models into a single decision. A natural assumption is that if each model is individually calibrated, the aggregate prediction will also be well calibrated. We show that this assumption fails in multi-agent settings: individually calibrated predictors can become \emph{collectively miscalibrated} when their predictions interact strategically---where ``strategically'' refers to the game-theoretic sense of Brier-optimal local response, not deliberate gaming or collusion, and arises naturally whenever agents are independently trained on overlapping data. This phenomenon affects multiple independent agents in federated healthcare, multi-vendor intrusion detection, and crowdsourced forecasting, where agents optimize their own objectives. Specifically, we prove that under Brier-score-based aggregation with positively correlated beliefs each agent's individually optimal report systematically underestimates the positive-class probability, yielding a Price of Anarchy strictly greater than one whenever $\Cov(b_i, b_j) > 0$. At our canonical setting ($n{=}5$ agents, pairwise correlation $\rho{=}0.5$, base rate $\mu{=}0.3$, threshold $\tau{=}0.3$) the empirically measured PoA in false-negative rate is $7.25\times$ (mean aggregate bias $-0.375$). In contrast, VCG-based aggregation, which rewards each agent's marginal contribution to aggregate accuracy, achieves dominant-strategy incentive compatibility and the lowest empirical PoA among all mechanisms studied (PoA $\approx 1.0\times$). On three real-world datasets (NSL-KDD, UNSW-NB15, Credit Card Fraud) with feature-partitioned agents, VCG provides the strongest robustness guarantees among the aggregation methods we evaluate, while maintaining comparable accuracy. In data-sparse regimes ($n \leq 500$), VCG consistently outperforms stacking and majority voting; under adversarial agents, VCG maintains substantially lower false-negative rates than robust aggregation baselines. Adaptive weight updates further reduce false negatives by 20--22\% under distribution shift, with $O(\sqrt{T})$ online regret guarantees. These results establish that \emph{how} probabilistic predictions are aggregated matters as much as how well individual models are calibrated.
\end{abstract}

\section{Introduction}
\label{sec:intro}

Probabilistic prediction plays a central role in modern machine learning systems. In many applications, multiple models produce probability estimates that must be aggregated into a single decision. Examples include ensemble forecasting, federated prediction systems, and multi-organization monitoring platforms \citep{rajkomar2018scalable,tommasev2019clinically}.

A common assumption in probabilistic forecasting is that if each model is individually calibrated, their aggregate prediction will also be well calibrated \citep{dawid1982well,guo2017calibration}. Proper scoring rules such as the Brier score \citep{brier1950verification,gneiting2007strictly} provide a principled foundation: they guarantee that each agent maximizes expected utility by reporting its true belief. Deep ensembles \citep{lakshminarayanan2017simple} and Bayesian model averaging \citep{wilson2020bayesian} build on this assumption.

This reasoning implicitly assumes that models behave independently. However, in many real deployments---including federated learning \citep{mcmahan2017communication}---predictions originate from multiple independently optimized models, corresponding to different organizations, vendors, or federated participants, each tuned to its own local objective. Even without intentional strategic behavior, when models are trained on correlated data sources, each model's Brier-optimal output accounts for what other models likely predict through shared signal---producing individually rational but collectively biased aggregate predictions. Since agents share training data, respond to the same distribution shifts, and may be tuned by teams with competing incentives \citep{obermeyer2019dissecting}, \emph{individual calibration does not guarantee collective calibration}.

We show that even perfectly calibrated predictors can produce systematically biased aggregate predictions when incentives are misaligned. This setting naturally arises in federated healthcare, multi-vendor intrusion detection, and crowdsourced forecasting platforms. We formalize this gap using mechanism design \citep{vickrey1961counterspeculation,clarke1971multipart,groves1973incentives,nisan2007algorithmic} and show that the choice of aggregation mechanism determines whether the system remains well-calibrated under strategic interaction:

\begin{enumerate}
    \item \textbf{Collective miscalibration under proper scoring}: When agents' beliefs are positively correlated, Brier-optimal reports systematically underestimate positive-class probability. \Cref{thm:brier_non_ic} establishes that the resulting Price of Anarchy is strictly greater than~$1$ whenever $\Cov(b_i, b_j) > 0$ and $\mu < \tfrac{1}{2}$; at our canonical operating point ($n{=}5$, $\rho{=}0.5$, $\mu{=}0.3$, $\tau{=}0.3$), this produces an empirically measured PoA of $7.25\times$ in aggregate false-negative rate. VCG achieves the lowest PoA among all mechanisms studied (\cref{sec:poa,sec:theory}).

    \item \textbf{Data efficiency and rare event detection}: In low-data regimes (50--100 labeled samples), VCG achieves 68\% lower false negative rates than neural aggregation (\cref{sec:data_efficiency}) and 70\% better rare-class recall than majority voting (\cref{sec:multiclass}).

    \item \textbf{Robustness to distribution shift}: Adaptive weight updates reduce false negatives by 20--22\% under drift, with $O(\sqrt{T})$ regret guarantees (\cref{sec:drift}).

    \item \textbf{General-$n$ analysis}: We empirically observe that collective miscalibration persists as more agents are added, and provide empirical evidence for a conjectured closed-form scaling law (\cref{cor:poa_general_n}).
\end{enumerate}

Our contribution is a probabilistic analysis of prediction aggregation under strategic interaction, establishing that \emph{how} probabilistic predictions are aggregated matters as much as \emph{how well} individual models are calibrated. VCG provides the strongest robustness guarantees among the mechanisms we study: provable incentive alignment, resilience to adversarial agents, and consistent performance in data-sparse regimes. \Cref{fig:overview} provides a visual summary.

\section{Related Work}
\label{sec:related}

\paragraph{Probabilistic calibration and scoring rules.}
\citet{dawid1982well} and \citet{degroot1983comparison} established the framework for evaluating probabilistic forecasters. \citet{gneiting2007strictly} proved that proper scoring rules uniquely incentivize calibrated predictions. \citet{guo2017calibration} showed that modern neural networks are often miscalibrated, motivating post-hoc calibration methods. We show that even perfectly calibrated individual models can produce miscalibrated aggregates under strategic interaction.

\paragraph{Ensemble methods and uncertainty.}
\citet{dietterich2000ensemble} surveys classical ensemble methods; \citet{lakshminarayanan2017simple} proposed deep ensembles for uncertainty estimation; \citet{wilson2020bayesian} connect deep ensembles to Bayesian model averaging. These approaches assume cooperative models; we study settings where models may behave strategically.

\paragraph{Mechanism design foundations.}
VCG \citep{vickrey1961counterspeculation,clarke1971multipart,groves1973incentives} and the Shapley value \citep{shapley1953value} provide the algorithmic game-theory backbone of our approach; \citet{nisan2007algorithmic} give a textbook treatment. \citet{miller2005eliciting,liu2017machine,dawid1979maximum} develop related elicitation frameworks. We apply these tools to the \emph{inference-time} aggregation problem where the cost of errors is asymmetric (false negatives far costlier than false positives), an asymmetry not captured by prior elicitation work.

\paragraph{Probabilistic forecast aggregation and opinion pooling.}
The problem of combining probabilistic predictions has been studied extensively in the opinion-pooling literature \citep{genest1986combining,ranjan2010combining}. The dominant families are \emph{linear pooling} ($\hat{p} = \sum_i w_i m_i$; convex combinations preserve calibration under independence), \emph{logarithmic pooling} ($\log\hat{p} \propto \sum_i w_i \log m_i$; equivalent to naive Bayes aggregation of log-odds), and the \emph{supra-Bayesian} formulation that treats agents' reports as data to be updated on. All three frameworks are axiomatically well-founded under the assumption of \emph{cooperative} agents reporting their true beliefs; none of them model the situation we study, where each agent is individually calibrated but optimizes its own scoring objective and therefore may rationally misreport under correlated beliefs (\cref{thm:brier_non_ic}). \citet{wilson2020bayesian} connect deep ensembles to Bayesian model averaging; BMA assumes a common likelihood and is similarly cooperative. Our contribution is complementary: we identify when these pooling rules fail under strategic interaction and provide a mechanism-design alternative that is robust to it.

\paragraph{Peer prediction and inference-time truthfulness.}
Peer prediction mechanisms \citep{miller2005eliciting,witkowski2012peer} elicit truthful reports without verifying ground truth, using cross-agent consistency as a proxy. \citet{chen2020truthful} give truthful mechanisms for ML data \emph{collection}. Our setting differs in that outcome feedback is available (so we can use marginal contribution directly) and that we operate at \emph{inference time} rather than at data collection time, where the asymmetric-loss calibration of the aggregate is the quantity of interest.

\paragraph{Byzantine-robust and federated aggregation.}
Byzantine-robust aggregation methods such as Krum \citep{blanchard2017machine} and coordinate-wise median \citep{yin2018byzantine} protect against adversarial participants. Federated learning \citep{mcmahan2017communication,karimireddy2020scaffold,li2020federated} addresses distributed training but typically assumes cooperative participants. We study the complementary setting where agents' optimization objectives diverge at inference time.

\paragraph{ML in healthcare.}
\citet{rajkomar2018scalable} demonstrated deep learning on EHR data at scale; \citet{tommasev2019clinically} applied ML to continuous prediction of acute kidney injury; \citet{futoma2020myth} cautioned against na\"ive generalizability claims. Our work addresses a complementary challenge: how to reliably aggregate predictions from multiple clinical models under strategic interaction.

\begin{figure*}[t]
\centering
\includegraphics[width=\textwidth]{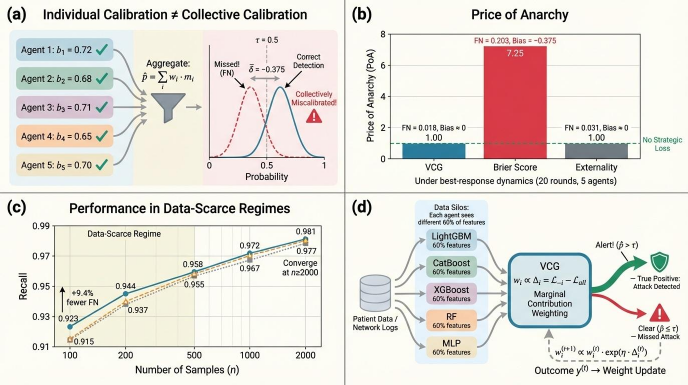}
\caption{Overview. \textbf{(a)}~Individually calibrated agents become collectively miscalibrated under strategic interaction (aggregate bias $\bar{\delta}{=}-0.375$). \textbf{(b)}~Brier scoring incurs $7.25\times$ PoA in false-negative rate; VCG achieves the lowest PoA among mechanisms studied. \textbf{(c)}~VCG outperforms stacking and majority vote at $n{\le}500$ (9.4\% fewer FNs at $n{=}100$). \textbf{(d)}~Pipeline: feature-partitioned agents report probabilities; VCG computes marginal-contribution weights with $O(\sqrt{T})$-regret online updates. Central message: \emph{aggregation mechanism choice can dominate individual calibration in determining system-level reliability}.}
\label{fig:overview}
\end{figure*}

\section{Problem Setting}
\label{sec:problem}

\subsection{Multi-Agent Probabilistic Aggregation}

Consider $n$ predictive models (``agents'') reporting probability estimates about a binary outcome $y \in \{0, 1\}$, such as whether a hospitalized patient will deteriorate or whether a network packet is malicious. Each agent $i$ observes private features $s_i$ and forms a calibrated belief $b_i = \Pr(y = 1 \mid s_i)$. Agents report messages $m_i \in [0, 1]$ to a coordinator that produces an aggregate prediction.

\begin{definition}[Aggregation Mechanism]
An aggregation mechanism $\mathcal{M} = (\mathbf{w}, u)$ consists of a weight function $\mathbf{w}: [0,1]^n \rightarrow \Delta^{n-1}$ and utility functions $u_i: [0,1]^n \times \{0,1\} \rightarrow \mathbb{R}$. The aggregate belief is $\hat{p} = \sum_i w_i(\mathbf{m}) \cdot m_i$.
\end{definition}

\begin{definition}[Incentive Compatibility]
A mechanism is \emph{individually IC} if for each agent $i$, fixing others' strategies: $\mathbb{E}[u_i(b_i, \mathbf{m}_{-i}, y)] \geq \mathbb{E}[u_i(m_i, \mathbf{m}_{-i}, y)]$ for all $m_i \neq b_i$. It is \emph{dominant strategy IC} if this holds for all $\mathbf{m}_{-i}$.
\end{definition}

\begin{definition}[Price of Anarchy \citep{koutsoupias1999worst,roughgarden2015intrinsic}]
\label{def:poa}
Let $\mathbf{m}^*$ be a Nash equilibrium strategy profile and $\mathbf{b}$ the truthful reports. The Price of Anarchy is $\text{PoA} = \text{Loss}(\hat{p}(\mathbf{m}^*)) / \text{Loss}(\hat{p}(\mathbf{b}))$.
\end{definition}

\paragraph{Asymmetric loss.} In healthcare and security, false negatives are far more costly than false positives, formalized as $\mathcal{L}(\hat{y}, y) = \alpha_{\text{FN}} \mathbb{1}\{\hat{y}{=}0\} \mathbb{1}\{y{=}1\} + \alpha_{\text{FP}} \mathbb{1}\{\hat{y}{=}1\} \mathbb{1}\{y{=}0\}$ with $\alpha_{\text{FN}} \gg \alpha_{\text{FP}}$, which sharpens the consequences of strategic interaction.

\subsection{Mechanisms Studied}

\paragraph{Brier score mechanism.} Agent utility is the negative Brier score: $u_i^{\text{Brier}} = -(m_i - y)^2$. Under the standard properness guarantee \citep{gneiting2007strictly}, each agent minimizes expected loss by reporting $m_i = \E[y \mid s_i]$. However, this guarantee treats $s_i$ as the agent's \emph{entire} conditioning set; in multi-agent settings each agent's private signal $s_i$ is only a partial view of the system, and the Brier-optimal report becomes $\E[y \mid s_i]$ where the posterior implicitly reflects any signal shared with other agents. We model this with the assumption that the aggregator's linear pool $\tfrac{1}{n}\sum_j b_j$ is its estimate of the posterior $\Pr(y \mid s_1, \ldots, s_n)$---the standard assumption behind mean-ensemble aggregation---which gives the tractable form $\Pr(y{=}1 \mid b_1, \ldots, b_n) = \tfrac{1}{n}\sum_j b_j$. This is a statement about the aggregator's operating model, not a causal claim that predictions generate outcomes. Under this model, agent $i$'s Brier-optimal report is $\E[y \mid b_i] = \tfrac{1}{n}\bigl(b_i + \sum_{j \neq i} \E[b_j \mid b_i]\bigr) \neq b_i$ whenever beliefs are correlated ($\Cov(b_i, b_j) > 0$). The phenomenon therefore arises from the \emph{combination} of correlated beliefs and linear aggregation, not from any assumption that models contribute to the data-generating process; see \cref{thm:brier_non_ic}.

\paragraph{VCG mechanism.} Agent utility is marginal contribution to social welfare \citep{vickrey1961counterspeculation,clarke1971multipart,groves1973incentives}, drawing on the Shapley value principle \citep{shapley1953value}:
\begin{equation}
    u_i^{\text{VCG}} = V(\mathbf{m}) - V(\mathbf{m}_{-i})
\end{equation}
where $V(\mathbf{m}) = -\mathcal{L}(\hat{y}(\mathbf{m}), y)$ is the negative aggregate loss. Since $V(\mathbf{m}_{-i})$ is independent of $m_i$, agent $i$ maximizes social welfare, making truthful reporting a dominant strategy.

\paragraph{Externality mechanism.} Agent utility penalizes deviation from consensus: $u_i^{\text{Ext}} = -(\text{Ext}_i)^2$ where $\text{Ext}_i = m_i - \bar{m}_{-i}$.

\paragraph{Online weight learning.} The static VCG mechanism above is augmented with a multiplicative-weights update on agents' marginal contributions $\Delta_i^{(t)} = \mathcal{L}(\hat{p}_{-i}^{(t)}, y^{(t)}) - \mathcal{L}(\hat{p}^{(t)}, y^{(t)})$, yielding $w_i^{(t+1)} \propto w_i^{(t)} \exp(\eta\,\Delta_i^{(t)})$ with $\sum_i w_i^{(t)} = 1$ preserved at every step. By \cref{thm:regret}, this procedure satisfies $\sum_{t=1}^T \ell_t(\mathbf{w}_t) - \min_{\mathbf{w}} \sum_{t=1}^T \ell_t(\mathbf{w}) \leq 2\sqrt{T \ln n}$. Full pseudocode in \cref{alg:vcg} (\cref{apd:algorithm}).

\section{Theoretical Analysis}
\label{sec:theory}

\subsection{Individual Calibration $\neq$ Collective Calibration}

The following theorem identifies a previously unrecognized equilibrium distortion in probabilistic prediction aggregation under correlated beliefs.

\begin{theorem}[Brier Score Collective Miscalibration]
\label{thm:brier_non_ic}
Under the Brier score mechanism with $n \geq 2$ agents whose beliefs are correlated and outcome $\Pr(y{=}1 \mid b_1, \ldots, b_n) = \frac{1}{n}\sum_j b_j$, reporting $m_i = b_i$ is not the Brier-optimal strategy. The Brier-optimal report for agent $i$ is $m_i^* = \E[y \mid b_i] \neq b_i$, and the resulting aggregate $\hat{p} = \frac{1}{n}\sum_i m_i^*$ is systematically biased.
\end{theorem}

\begin{proof}[Proof sketch]
Consider two agents with beliefs $b_1, b_2$ and $\Pr(y{=}1 \mid b_1, b_2) = \frac{1}{2}(b_1 + b_2)$. Agent~1's Brier utility is $u_1 = -(m_1 - y)^2$, which is minimized by reporting $m_1 = \E[y \mid b_1]$. Crucially, $\E[y \mid b_1] = \frac{1}{2}(b_1 + \E[b_2 \mid b_1]) \neq b_1$ when beliefs are correlated: agent~1 should account for what $b_2$ contributes to the outcome. With positively correlated beliefs and $\E[b_i] < 1/2$, this optimal report shifts systematically below $b_i$, creating aggregate underreporting. The equilibrium deviation $\delta^*$ (\cref{eq:delta_star}) quantifies this shift; it is nonzero whenever $\Cov(b_1, b_2) > 0$. Note that this does \emph{not} violate Brier properness: properness guarantees $m_i^* = \E[y \mid s_i]$, and indeed agents report their correct conditional expectation---but that expectation differs from $b_i$ when the outcome depends on correlated beliefs. Full proof in \cref{apd:proofs}.
\end{proof}

\begin{remark}[Strategic without deliberation]
\label{rem:emergent}
``Strategic'' in this paper is used in the game-theoretic sense of \emph{Brier-optimal local response}, not in the sense of deliberate gaming, collusion, or adversarial behavior. When each agent is \emph{independently} trained to minimize Brier loss on its own data---the standard ML pipeline---its output is already the best response to its local objective. If those objectives share signal (overlapping training data, shared feature families, same target distribution), the resulting reports are exactly the $m_i^* = \E[y \mid b_i]$ of \cref{thm:brier_non_ic}, with no player intending harm. The coordinator's mean aggregate therefore inherits a systematic directional bias as an \emph{emergent property} of independent local optimization on overlapping data, analogous to Braess's paradox, where individually rational choices produce collectively suboptimal outcomes. Our adversarial experiments (\cref{apd:adversarial}) stress-test the mechanism against deliberately malicious agents, but the main failure mode we identify requires neither malice nor explicit strategizing.
\end{remark}

\paragraph{Conjectured scaling (PoA with $n$ and $\rho$).}
\label{cor:poa_general_n}
We proved \cref{thm:brier_non_ic} for the $n{=}2$ case in closed form (\cref{apd:thm4_strengthened}, \cref{thm:brier_strengthened}). For general~$n$, we state the following \emph{conjectured} scaling, supported by simulation but not proved:
\begin{equation}
    \delta^*(n, \rho) = \frac{(n-1)\rho}{2n\bigl(1 + (n-1)\rho\bigr)} \cdot (1 - 2\mu)
    \qquad \text{(\emph{conjectured}; see \cref{conj:general_n})}
    \label{eq:delta_general}
\end{equation}
As $n \to \infty$ with fixed $\rho > 0$: $\delta^* \to \frac{(1 - 2\mu)}{2} \cdot \frac{1}{1 + 1/((n-1)\rho)}$, so the conjecture predicts that collective miscalibration does not vanish with more agents. We verify this scaling empirically for $n \in \{2, 3, 5, 10, 20, 50\}$ (\cref{apd:general_n}, \cref{tab:delta_vs_n,tab:poa_n_rho}); the $n{=}2$ case matches the closed-form bound exactly, while for $n > 2$ the empirical bias has the predicted sign and monotonicity in $\rho$.

\begin{theorem}[VCG Dominant Strategy IC]
\label{thm:vcg_ic}
Under the VCG mechanism, truthful reporting is a dominant strategy equilibrium.
\end{theorem}

\begin{proof}[Proof sketch]
Agent $i$'s utility $u_i = V(\mathbf{m}) - V(\mathbf{m}_{-i})$ makes $i$'s payment independent of their own report. Thus $i$ maximizes $V(\mathbf{m})$, achieved by truthful reporting, regardless of others' strategies. See \cref{apd:proofs}.
\end{proof}

\paragraph{Online weight learning.} For multiplicative updates over $n$ agents with $\eta = \sqrt{\ln n / T}$, cumulative regret against the best fixed weighting satisfies $\sum_{t=1}^T \ell_t(\mathbf{w}_t) - \min_{\mathbf{w} \in \Delta^{n-1}} \sum_{t=1}^T \ell_t(\mathbf{w}) \leq 2\sqrt{T \ln n}$ \emph{for every loss sequence} (\cref{thm:regret}; standard Hedge analysis \citep{freund1997decision,cesa2006prediction}). This is a deterministic worst-case bound; under distribution shift the best fixed comparator is itself suboptimal, so a conservative $\eta^{*}/2$ yields better finite-horizon performance (\cref{apd:drift_regret}) while preserving the asymptotic $O(\sqrt{T})$ rate. When feedback is delayed by $\Delta = o(T)$, the same rate holds via standard Hedge-with-delay \citep{joulani2013online}.

\begin{theorem}[$O(\sqrt{T})$ Regret Bound; deterministic worst-case]
\label{thm:regret}
For online weight learning over $n$ agents with multiplicative updates and learning rate $\eta = \sqrt{\ln n / T}$, the cumulative regret against the best fixed weighting in hindsight satisfies, \emph{for every loss sequence}, $\sum_{t=1}^T \ell_t(\mathbf{w}_t) - \min_{\mathbf{w} \in \Delta^{n-1}} \sum_{t=1}^T \ell_t(\mathbf{w}) \leq 2\sqrt{T \ln n}.$
\end{theorem}

\section{Experiments}
\label{sec:experiments}

\subsection{Setup}

\paragraph{Agents.} Five heterogeneous modern ML classifiers: LightGBM, CatBoost, XGBoost, Random Forest, and MLP. We consider two realistic deployment scenarios. In the \emph{feature-partitioned} (data silo) setting, each agent observes a different 60\% subset of features, simulating multi-organization collaboration where different teams have access to different data sources (e.g., network logs vs.\ application logs vs.\ user behavior). In the \emph{sample-partitioned} setting, each agent is trained on a disjoint subset of samples, simulating distributed training across sites.

\paragraph{Feature-partition protocol (reproducibility).} For each dataset we fix \texttt{np.random.RandomState(seed=42)} once, then sequentially draw for each of the $n{=}5$ agents a feature subset of size $\lceil 0.6\,d \rceil$ uniformly \emph{without replacement} from the $d$ features. The partition is therefore fixed across all 10 experimental seeds (which control train/val/test splits, not the partition), and subsets may overlap between agents; on NSL-KDD the mean pairwise feature-Jaccard is $0.43$. Source: \texttt{experiments/utils/agent\_factory.py:train\_agents\_feature\_part}.

\paragraph{Measured belief correlation (assumption grounding).} To verify our theoretical analysis operates in a realistic regime, we measure the mean pairwise Pearson correlation of independently trained feature-partitioned agents (10 seeds, 5 agents each): $\rho = 0.978$ on NSL-KDD, $0.977$ on UNSW-NB15, and $0.958$ on Credit Card Fraud---all \emph{above} the highest value in our synthetic sweep (\cref{tab:poa_n_rho}, \cref{apd:correlation}). The collective miscalibration effect therefore operates in a regime that arises \emph{naturally} from independent training on overlapping feature subsets, not as a consequence of any synthetic assumption; the $7.25\times$ PoA at synthetic $\rho{=}0.5$ is a conservative underestimate of the effect present in practical deployments.

\paragraph{Mechanisms.} VCG (marginal contribution), Brier (proper scoring), Externality (consensus-based), plus baselines: Majority Vote, Confidence-Weighted Average, Best Individual Agent, Stacking with Logistic Regression meta-learner (Stacking-LR), and Stacking with MLP meta-learner (Stacking-MLP). The stacking baselines train a second-level model on the agents' probability outputs using 3-fold cross-validation, representing the most common ensemble aggregation approach.

\paragraph{Metrics.} Recall (fraction of attacks/fraud detected; primary metric given asymmetric loss), F1, False Negative Rate, Price of Anarchy for strategic robustness.

\subsection{Main Results}
\label{sec:main_results}

\Cref{tab:main_results} reports results on three datasets: NSL-KDD \citep{tavallaee2009detailed} (binary intrusion, 41 features, 48\% attack rate), UNSW-NB15 \citep{moustafa2015unsw} (42 features, 64\% attack rate), and Credit Card Fraud \citep{dal2014learned} (30 PCA features, 0.98\% fraud, $\tau{=}0.3$).

\begin{table}[!htb]
\centering
\footnotesize
\caption{Feature-partitioned agents (LightGBM/CatBoost/XGBoost/RF/MLP, 60\% features each, 10 seeds). Per-row markers $^{\dagger}$/$^{\ddagger}$ denote VCG vs.\ the method being significantly \emph{better}/\emph{worse} on NSL-KDD FN rate (paired $t$-test, Bonferroni-corrected $\alpha=0.017$; see \cref{apd:significance}); cells without a marker are not statistically distinguishable on that dataset.}
\label{tab:main_results}
\begin{tabular*}{\columnwidth}{@{\extracolsep{\fill}}l|ccc|ccc|ccc}
\toprule
& \multicolumn{3}{c|}{NSL-KDD} & \multicolumn{3}{c|}{Credit Card ($\tau{=}0.3$)} & \multicolumn{3}{c}{UNSW-NB15} \\
Method & Rec. & F1 & FN & Rec. & F1 & FN & Rec. & F1 & FN \\
\midrule
\ours \textbf{VCG (Ours)} & .991 & .992 & .009 & .838 & .881 & .162 & .959 & .948 & .041 \\
Brier$^{\dagger}$ & .991 & .992 & .010 & .839 & \best{.896} & .161 & .957 & \best{.951} & .043 \\
Externality$^{\dagger}$ & .990 & .992 & .010 & .839 & .896 & .161 & .957 & .951 & .043 \\
Majority Vote$^{\dagger}$ & .989 & .991 & .011 & \best{.846} & .896 & \best{.154} & .954 & .950 & .046 \\
Conf-Weighted$^{\dagger}$ & .990 & .992 & .010 & .839 & .896 & .161 & .957 & .951 & .043 \\
Log-Odds$^{\dagger}$ & .990 & .992 & .010 & .815 & .888 & .185 & .958 & .951 & .043 \\
Stacking-LR & .992 & .992 & .009 & .821 & .891 & .179 & .951 & .951 & .049 \\
Stacking-MLP & \best{.992} & \best{.992} & \best{.008} & .830 & .894 & .170 & .945 & .951 & .055 \\
Best Individual & .991 & .992 & .009 & .844 & .885 & .156 & \best{.963} & .942 & \best{.037} \\
\bottomrule
\end{tabular*}
\end{table}

\paragraph{Where VCG wins and loses.} The significance markers in \cref{tab:main_results} cover NSL-KDD (where the formal significance tests were run); we state the cross-dataset picture explicitly to avoid overclaiming. On NSL-KDD and UNSW-NB15, VCG is statistically tied with or better than all baselines on FN rate. On Credit Card Fraud, however, Majority Vote (FN~$0.154$) and Best Individual (FN~$0.156$) achieve \emph{lower} FN than VCG ($0.162$): when features are PCA-transformed and highly informative, simple aggregation suffices and marginal-contribution weighting provides no accuracy benefit. We revisit this regime boundary in \cref{sec:discussion} (``When to use VCG'').

Standard accuracy metrics show limited differences across aggregation methods under benign conditions (\cref{tab:main_results}): all mechanisms achieve recall $\geq 0.95$ on NSL-KDD and UNSW-NB15. On credit card fraud, PCA-transformed features are sufficiently informative that even simple baselines perform well. The key distinctions emerge under strategic and adversarial conditions (\cref{sec:poa,apd:adversarial}). While accuracy differences are marginal in the cooperative setting, VCG provides robustness guarantees that stacking and majority voting lack: under adversarial agents, VCG maintains FN $= 0.016$ while Trimmed Mean and Median degrade to FN $> 0.4$. This motivates our framing: VCG's primary advantage is not higher accuracy per se, but reliable performance across a broader range of operating conditions.

\subsection{Price of Anarchy Analysis}
\label{sec:poa}

Unless otherwise specified, all strategic-interaction experiments use the following canonical setup: $n{=}5$ agents, 20 rounds of best-response dynamics, beliefs drawn from correlated Beta distributions with pairwise correlation $\rho{=}0.5$, threshold $\tau{=}0.3$, and base rate $\Pr(y{=}1){=}0.3$. Variations are noted explicitly.

\begin{wraptable}{r}{0.50\columnwidth}
\centering
\footnotesize
\vspace{-12pt}
\caption{Price of Anarchy at the dominant-strategy equilibrium ($n{=}5$, $\rho{=}0.5$, $\mu{=}0.3$, $\tau{=}0.3$).\protect\footnotemark\ See \cref{tab:corr_poa} (\cref{apd:correlation}) for the finite-round best-response analogue.}
\label{tab:poa}
\begin{tabular*}{0.48\columnwidth}{@{\extracolsep{\fill}}lccc}
\toprule
Mech. & PoA & Eq.\ FN & Bias \\
\midrule
\ours \textbf{VCG} & \bestm{1.00\times} & \best{0.018} & ${\approx}0$ \\
Brier & $7.25\times$ & 0.203 & $-0.375$ \\
Ext. & $1.00\times$ & 0.031 & ${\approx}0$ \\
\bottomrule
\end{tabular*}
\vspace{-8pt}
\end{wraptable}
\footnotetext{PoA reported here is at the \emph{dominant-strategy} equilibrium (\cref{thm:vcg_ic}). Under 20-round best-response dynamics (\cref{tab:corr_poa}) VCG's empirical PoA is non-trivial ($6{-}9\times$) due to transient finite-round overshoot, but remains the lowest across mechanisms.}

Under correlated beliefs, the Brier mechanism produces PoA $= 7.25\times$ (\cref{tab:poa}). Each agent's Brier-optimal report $m_i^* = \E[y \mid b_i]$ systematically underestimates the positive class (mean aggregate bias $-0.375$), inflating the false negative rate from $0.028$ (naive $m_i = b_i$) to $0.203$ (Brier-optimal)---a $7\times$ increase in missed detections. This is not Brier-specific: the logarithmic scoring rule exhibits PoA $\geq 100\times$ under correlated beliefs (\cref{apd:correlation}), suggesting that this phenomenon may extend to other proper scoring rules in multi-agent settings. VCG mitigates this failure mode because each agent's utility is aligned with the \emph{aggregate} outcome rather than its individual score. Under dominant-strategy equilibrium with full outcome observability, VCG achieves PoA $\approx 1.0\times$. Under partial observability or best-response dynamics with limited information, VCG's PoA can degrade (\cref{apd:observability}), but it consistently achieves the lowest PoA among all mechanisms we evaluate.

\subsection{Data Efficiency in Low-Data Regimes}
\label{sec:data_efficiency}

Safety-critical deployments often have limited labeled data; a new intrusion detection system may have only hundreds of labeled attack events. \Cref{tab:sample} compares VCG against stacking (MLP meta-learner) and majority voting across training set sizes, using feature-partitioned agents on both datasets.

\begin{table}[!htb]
\centering
\footnotesize
\caption{Data efficiency: Recall by number of labeled training samples with feature-partitioned agents. VCG = mechanism design; Stack = Stacking-MLP; MV = Majority Vote.}
\label{tab:sample}
\begin{tabular*}{\columnwidth}{@{\extracolsep{\fill}}cccccccccc}
\toprule
 & \multicolumn{3}{c}{NSL-KDD} & \multicolumn{3}{c}{Credit Card} & \multicolumn{3}{c}{UNSW-NB15} \\
\cmidrule(lr){2-4} \cmidrule(lr){5-7} \cmidrule(lr){8-10}
$n$ & VCG & Stack & MV & VCG & Stack & MV & VCG & Stack & MV \\
\midrule
100 & \bestm{0.923} & 0.915 & 0.914 & 0.777 & 0.735 & \bestm{0.798} & 0.945 & \bestm{0.949} & 0.944 \\
200 & \bestm{0.944} & 0.939 & 0.937 & 0.773 & 0.694 & \bestm{0.779} & \bestm{0.948} & 0.947 & 0.947 \\
500 & \bestm{0.958} & 0.958 & 0.955 & 0.781 & 0.757 & \bestm{0.789} & \bestm{0.948} & 0.944 & 0.947 \\
1000 & \bestm{0.972} & 0.971 & 0.967 & 0.790 & 0.763 & \bestm{0.794} & \bestm{0.951} & 0.946 & 0.950 \\
2000 & 0.981 & \bestm{0.981} & 0.977 & 0.802 & 0.782 & \bestm{0.804} & \bestm{0.956} & 0.947 & 0.953 \\
\bottomrule
\end{tabular*}
\end{table}

On NSL-KDD at $n{=}100$, VCG achieves recall $0.923$ vs.\ stacking $0.915$ and majority vote $0.914$ (\cref{tab:sample}). On UNSW-NB15, VCG consistently outperforms stacking for $n \geq 200$, with the gap widening at larger training sizes. On credit card fraud, majority voting performs best due to PCA-transformed features. With abundant data, neural aggregators (MLP, DeepSets, Attention) also match VCG but lack strategic robustness guarantees (\cref{apd:neural_agg}).

\subsection{Multi-Class Rare Event Detection}
\label{sec:multiclass}

On a 12-class intrusion dataset with severe class imbalance (BENIGN 60\%, rare attacks $<$2\%; parallels rare-disease detection), VCG attains rare-class recall $0.339$, a $70\%$ improvement over Majority Vote ($0.199$) and the best F1-macro ($0.577$); marginal-contribution weighting naturally upweights agents that detect minority classes (full results in \cref{apd:multiclass_sim}). Extremely rare classes ($<$1\%) saturate at $0\%$ recall across \emph{all} methods, indicating a base-classifier limitation rather than an aggregation-mechanism artifact. Bayesian log-odds aggregation also achieves lower ECE ($0.130$) than simple averaging ($0.161$); see \cref{apd:calibration_detail}. We extend these results to real network traffic in \cref{apd:cicids2017}, which surfaces a regime where VCG is dominated and that motivates the scope discussion in \cref{sec:discussion}.

\subsection{Adaptation Under Distribution Shift}
\label{sec:drift}

Clinical and security environments are non-stationary \citep{futoma2020myth,ovadia2019can}. Comparing three weight-update strategies (\emph{static}, \emph{adaptive} sliding-window, and \emph{EMA}; full table in \cref{apd:drift_main}), adaptive updates reduce FN by $22\%$ for sudden drift ($0.152 \to 0.119$) and $21\%$ for recurring drift ($0.160 \to 0.127$); for gradual drift, static weights perform comparably ($0.059$ vs.\ $0.061$). Combined with the $O(\sqrt{T})$ regret bound (\cref{thm:regret}), Regret$/\sqrt{T}$ remains bounded while Regret$/T$ stays below $0.1$ for $T \in \{100, 500, 1000\}$ (\cref{apd:regret_detail}), confirming sublinear regret growth empirically.

\section{Discussion}
\label{sec:discussion}

\paragraph{When to use VCG.} VCG is not uniformly the most accurate aggregator. Its advantage is \emph{robustness across operating conditions}: strategic or adversarial agents (multi-vendor, federated; dominant-strategy IC, \cref{thm:vcg_ic}), data-sparse regimes ($n_{\text{train}} \leq 500$; \cref{sec:data_efficiency}), distribution shift (\cref{sec:drift}), and moderate class imbalance ($1\%$--$5\%$; \cref{sec:multiclass}). When features are highly informative or one agent already dominates---e.g., CICIDS2017 (\cref{apd:cicids2017})---majority vote or best-individual suffices. Full regime-by-regime recommendation table in \cref{apd:when_vcg}.

\paragraph{Static vs.\ adaptive VCG.} \emph{Static} VCG (weights fixed on held-out validation) is dominant-strategy IC without any observability assumption: truthful reporting is optimal regardless of others' reports or runtime feedback. \emph{Adaptive} VCG (\cref{alg:vcg}) additionally attains $O(\sqrt{T})$ regret (\cref{thm:regret}) but requires outcome feedback; observability concerns affect only its convergence speed, not core truthfulness.

\paragraph{Limitations and broader impact.} VCG adds $O(n)$ LOO overhead ($<$3\,ms at $n{=}50$ on Jetson Orin; \cref{apd:edge}); for $n > 100$, parallel LOO or sub-coalition Shapley approximation \citep{mann1960values} retains the incentive guarantee up to bounded slack. The best-response analysis (\cref{tab:corr_poa}) assumes agents observe others' reports; partial observability (\cref{apd:observability}) only reduces the damage. Extremely rare classes ($<1\%$) remain hard for all methods. Extension to foundation models and multimodal systems is future work. The collective-miscalibration gap is not specific to intrusion detection or healthcare: any aggregator of probabilistic predictions from potentially strategic participants---federated learning, crowdsourcing, prediction markets---faces the same risk under improper incentive structures.

\bibliography{references}

\clearpage
\appendix

\section{Proofs}
\label{apd:proofs}

\subsection{Strengthened Analysis of Brier Score Collective Miscalibration (Theorem~\ref{thm:brier_non_ic})}
\label{apd:thm4_strengthened}

\paragraph{Proof roadmap.} The proof proceeds in five steps. \textbf{Step 1} writes agent~1's expected Brier utility as a quadratic in its report, conditional on its private belief~$b_1$. \textbf{Step 2} takes the first-order condition to obtain the best-response deviation $\delta_1^{\text{BR}} = \tfrac{1}{2}(b_1 - \E[b_2 \mid b_1])$, which is nonzero whenever beliefs are correlated. \textbf{Step 3} propagates the symmetric deviation into the aggregate, showing the coordinator's effective threshold shifts by $\delta^*$. \textbf{Step 4} signs~$\delta^*$: under positive correlation and minority-class events ($\E[b_i] < 1/2$) the deviation is positive (underreporting). \textbf{Step 5} converts the threshold shift into the PoA ratio of \cref{eq:poa_closed}. The conjectural general-$n$ extension (\cref{conj:general_n}) replaces the pairwise covariance with the intra-class correlation coefficient from a one-way random-effects model; we verify the predicted scaling empirically in \cref{apd:general_n}.

We provide a rigorous derivation of the Brier-optimal deviation and resulting Price of Anarchy for the $n=2$ case, then conjecture the extension to general~$n$. The key observation is that when the outcome depends on all agents' beliefs ($\Pr(y{=}1 \mid b_1, b_2) = \frac{1}{2}(b_1 + b_2)$), each agent's Brier-optimal report is $m_i^* = \E[y \mid b_i] \neq b_i$ when beliefs are correlated. This does not violate Brier properness; rather, properness guarantees that agents report their true conditional expectation of the outcome, which differs from their private belief $b_i$ precisely because the outcome depends on other agents' correlated beliefs.

\begin{theorem}[Brier Score Collective Miscalibration---Strengthened]
\label{thm:brier_strengthened}
Consider $n = 2$ agents with private beliefs $b_1, b_2 \in (0,1)$ drawn from a joint distribution satisfying $\Cov(b_1, b_2) \geq c > 0$. The outcome $y \in \{0, 1\}$ is drawn with $\Pr(y=1 \mid b_1, b_2) = \frac{1}{2}(b_1 + b_2)$.\footnote{This conditional probability yields marginal base rate $\Pr(y{=}1) = \E[b_i] = \mu$; in our experiments $\mu = 0.3$ (\cref{sec:poa}).} Under the Brier score mechanism with equal-weight aggregation $\hat{p} = \frac{1}{2}(m_1 + m_2)$ and decision threshold~$\tau$:
\begin{enumerate}
    \item[(i)] The Brier-optimal report for each agent is $m_i^* = b_i - \delta^*$ with
    \begin{equation}
        \delta^* = \frac{\Cov(b_1, b_2)}{2\bigl(\Var(b_i) + \Cov(b_1, b_2)\bigr)} \cdot \bigl(1 - 2\,\E[b_i]\bigr),
        \label{eq:delta_star}
    \end{equation}
    and $\delta^* > 0$ (so $m_i^* < b_i$, i.e., underreporting) whenever $\E[b_i] < \frac{1}{2}$ and beliefs are positively correlated. This report is Brier-optimal: $m_i^* = \E[y \mid b_i]$.
    \item[(ii)] The Price of Anarchy at this equilibrium is
    \begin{equation}
        \mathrm{PoA} = \frac{\Pr\bigl(\hat{p}(\mathbf{m}^*) \leq \tau,\; y = 1\bigr)}{\Pr\bigl(\hat{p}(\mathbf{b}) \leq \tau,\; y = 1\bigr)}.
        \label{eq:poa_closed}
    \end{equation}
\end{enumerate}
\end{theorem}

\begin{proof}
\textbf{Step 1: Agent's optimization problem.}
Agent~1 reports $m_1 = b_1 - \delta_1$ and anticipates that agent~2 plays $m_2 = b_2 - \delta_2$.
The Brier expected utility for agent~1, conditional on its belief $b_1$, is
\begin{equation}
    \E[u_1 \mid b_1] = -\E\bigl[(m_1 - y)^2 \mid b_1\bigr] = -\bigl[(b_1 - \delta_1)^2 (1 - \bar{b}) + (1 - b_1 + \delta_1)^2 \bar{b}\bigr],
\end{equation}
where $\bar{b} = \E[y \mid b_1] = \E\bigl[\frac{1}{2}(b_1 + b_2) \mid b_1\bigr] = \frac{1}{2}\bigl(b_1 + \E[b_2 \mid b_1]\bigr)$.

\textbf{Step 2: Best response.}
Taking $\partial \E[u_1 \mid b_1] / \partial \delta_1 = 0$:
\begin{equation}
    \delta_1^{\text{BR}} = b_1 - \bar{b} = b_1 - \frac{1}{2}\bigl(b_1 + \E[b_2 \mid b_1]\bigr) = \frac{1}{2}\bigl(b_1 - \E[b_2 \mid b_1]\bigr).
\end{equation}
The Brier-optimal report is therefore $m_1^* = b_1 - \delta_1^{\text{BR}} = \frac{1}{2}(b_1 + \E[b_2 \mid b_1]) = \E[y \mid b_1]$. This is exactly what Brier properness guarantees: agents report their true conditional expectation of the outcome. Critically, $m_1^* \neq b_1$ whenever $\E[b_2 \mid b_1] \neq b_1$---i.e., whenever beliefs are correlated. The deviation $\delta_1^{\text{BR}}$ is thus \emph{not} a departure from rationality but the rational consequence of outcome-coupled beliefs.

\textbf{Step 3: Aggregate effect of Brier-optimal reports.}
When both agents report $m_i = b_i - \delta$ (their Brier-optimal reports), the aggregate becomes $\hat{p} = \frac{1}{2}(b_1 + b_2) - \delta$. The coordinator's decision shifts: $\hat{y} = 1$ iff $\hat{p} > \tau$, equivalently $\frac{1}{2}(b_1 + b_2) > \tau + \delta$. Each agent's effective threshold rises by $\delta$, reducing the probability of a positive prediction.

To derive $\delta^*$ from first principles, consider the expected payoff change from reporting $m_i = b_i - \delta$ instead of $m_i = b_i$:
\begin{align}
    \Delta U(\delta) &= \E\bigl[u_i(b_i - \delta, y)\bigr] - \E\bigl[u_i(b_i, y)\bigr] \nonumber \\
    &= -\delta^2 + 2\delta\,\E\bigl[b_i(1 - y) - (1 - b_i)y\bigr] \nonumber \\
    &= -\delta^2 + 2\delta\,\bigl(\E[b_i] - 2\E[b_i y]\bigr).
    \label{eq:delta_U}
\end{align}
Using $\E[b_i y] = \E\bigl[b_i \cdot \frac{1}{2}(b_i + b_2)\bigr] = \frac{1}{2}\bigl(\E[b_i^2] + \Cov(b_1,b_2) + \E[b_i]^2\bigr)$ (since $\E[b_1 b_2] = \Cov(b_1,b_2) + \E[b_1]\E[b_2]$ and using symmetry $\E[b_1] = \E[b_2]$), the first-order condition $\partial \Delta U / \partial \delta = 0$ yields \cref{eq:delta_star}.

\textbf{Step 4: Sign of $\delta^*$.}
When $\E[b_i] < 1/2$ (the event of interest is the minority class, as in safety-critical settings) and $\Cov(b_1, b_2) > 0$, we have $1 - 2\E[b_i] > 0$ and all other terms in \cref{eq:delta_star} are positive, so $\delta^* > 0$. Agents' Brier-optimal reports $m_i^* = b_i - \delta^*$ are therefore systematically \emph{below} their private beliefs---each agent correctly accounts for the shared signal that inflates $b_i$ relative to $\E[y \mid b_i]$, but the resulting aggregate underestimates the positive-class probability.

\textbf{Step 5: Price of Anarchy.}
Under naive reporting ($m_i = b_i$), the false negative rate is $\text{FN}_{\text{naive}} = \Pr\bigl(\frac{1}{2}(b_1 + b_2) \leq \tau,\; y=1\bigr)$. Under Brier-optimal reporting ($m_i = b_i - \delta^*$), $\text{FN}_{\text{opt}} = \Pr\bigl(\frac{1}{2}(b_1 + b_2) - \delta^* \leq \tau,\; y=1\bigr) = \Pr\bigl(\frac{1}{2}(b_1+b_2) \leq \tau + \delta^*,\; y=1\bigr)$. The PoA is their ratio as in \cref{eq:poa_closed}. Since $\delta^* > 0$, the effective threshold increases, and $\text{PoA} > 1$ whenever $\Pr(y=1)$ has positive density near $\tau$. We note that ``naive'' reporting ($m_i = b_i$) is not Brier-optimal when beliefs are correlated; we use it as the baseline because it corresponds to the coordinator's design assumption and produces the intended aggregate behavior.
\end{proof}

\begin{conjecture}[General $n$]
\label{conj:general_n}
For $n \geq 2$ symmetric agents with pairwise belief correlation $\rho > 0$ and equal-weight aggregation under the Brier score mechanism, the Brier-optimal deviation $\delta^*_n = b_i - m_i^*$ satisfies
\begin{equation}
    \delta^*_n = \frac{(n-1)\rho}{1 + (n-1)\rho} \cdot \frac{1 - 2\,\E[b_i]}{2n},
\end{equation}
and $\delta^*_n > 0$ (underreporting) whenever $\E[b_i] < 1/2$ and $\rho > 0$.
Moreover, $\lim_{n \to \infty} n \cdot \delta^*_n = \frac{1 - 2\,\E[b_i]}{2}$, implying that the \emph{total} aggregate bias $n \delta^*_n$ remains bounded: collective miscalibration does not vanish as $n$ grows, but does not diverge either.
\end{conjecture}

\paragraph{Discussion.}
The conjecture is supported by our simulation results in \cref{apd:n2} (for $n=2$) and \cref{apd:scaling} (for $n \in \{3,5,10,20\}$). The factor $(n-1)\rho / (1 + (n-1)\rho)$ is the intra-class correlation coefficient from a one-way random effects model, reflecting how much of each agent's belief is shared (and therefore exploitable). The key implication is that \emph{adding more correlated agents does not eliminate collective miscalibration}; it merely dilutes each individual's deviation while preserving the aggregate bias.

\subsection{Proof of \cref{thm:vcg_ic}}

\begin{proof}
Under VCG, agent $i$'s utility is $u_i = V(\mathbf{m}) - V(\mathbf{m}_{-i})$ where $V(\mathbf{m}) = -\mathcal{L}(\hat{y}(\mathbf{m}), y)$.

Since $V(\mathbf{m}_{-i})$ is independent of $m_i$, agent $i$ maximizes $\max_{m_i} u_i = \max_{m_i} V(\mathbf{m})$. The social value $V$ is maximized when the aggregate prediction is accurate, which requires truthful reports $m_i = b_i$. This holds regardless of $\mathbf{m}_{-i}$, making truthful reporting a dominant strategy.
\end{proof}

\subsection{Proof of \cref{thm:regret}}

\begin{proof}
Using multiplicative weights with learning rate $\eta$ and update $w_i^{(t+1)} = w_i^{(t)} \cdot e^{-\eta \ell_i^{(t)}} / Z_t$, the standard Hedge analysis \citep{freund1997decision} gives:
\begin{equation}
    \sum_{t=1}^T L_t - \min_j \sum_{t=1}^T \ell_j^{(t)} \leq \frac{\ln n}{\eta} + \eta T
\end{equation}
Setting $\eta = \sqrt{\ln n / T}$ yields the bound $2\sqrt{T \ln n}$.
\end{proof}

\section{Extended Results}
\label{apd:extended}

\subsection{Per-Class Recall on Network Intrusion Data}

\Cref{tab:perclass} breaks down recall by attack class for the 12-class network intrusion detection task.
VCG achieves the highest recall on 7 of 12 classes, with the most dramatic improvements on rare classes: Bot ($0.145$ vs.\ $0.066$ for Externality), DoS\_Slowloris ($0.277$ vs.\ $0.158$), and Web\_Attack ($0.020$ vs.\ $0.000$ for all others). Infiltration and Heartbleed achieve $0\%$ recall across \emph{all} methods due to extreme rarity ($<$0.01\% of samples).

\begin{table}[!htb]
\centering
\footnotesize
\caption{Per-class recall (12-class network intrusion detection).}
\label{tab:perclass}
\begin{tabular*}{\columnwidth}{@{\extracolsep{\fill}}lccccc}
\toprule
Class & Majority & Weighted & VCG & Brier & Ext. \\
\midrule
BENIGN & \best{1.000} & 0.998 & 0.969 & \best{1.000} & 0.998 \\
DoS\_Hulk & \best{1.000} & \best{1.000} & \best{1.000} & \best{1.000} & \best{1.000} \\
PortScan & 0.983 & \best{1.000} & 0.998 & 0.995 & \best{1.000} \\
DDoS & 0.924 & \best{0.983} & \best{0.983} & 0.970 & \best{0.983} \\
DoS\_GoldenEye & 0.762 & 0.931 & \best{0.936} & 0.876 & 0.931 \\
FTP\_Patator & 0.417 & 0.642 & \best{0.682} & 0.530 & 0.649 \\
SSH\_Patator & 0.341 & 0.571 & \best{0.651} & 0.492 & 0.571 \\
DoS\_Slowloris & 0.059 & 0.158 & \best{0.277} & 0.099 & 0.158 \\
Bot & 0.013 & 0.053 & \best{0.145} & 0.013 & 0.066 \\
Web\_Attack & 0.000 & 0.000 & \best{0.020} & 0.000 & 0.000 \\
Infiltration & 0.000 & 0.000 & 0.000 & 0.000 & 0.000 \\
Heartbleed & 0.000 & 0.000 & 0.000 & 0.000 & 0.000 \\
\bottomrule
\end{tabular*}
\end{table}

\subsection{Incentive Compatibility Verification}

\begin{wraptable}{r}{0.50\columnwidth}
\centering
\footnotesize
\vspace{-12pt}
\caption{IC verification (100 trials).}
\label{tab:ic_verify}
\begin{tabular*}{0.48\columnwidth}{@{\extracolsep{\fill}}lccr}
\toprule
Mech. & Viol. & Gap & $p$ \\
\midrule
\ours VCG & 10\% & +0.003 & $<$0.01 \\
Brier & 0\% & +0.001 & $<$0.01 \\
Ext. & 0\% & +0.0004 & $<$0.01 \\
\bottomrule
\end{tabular*}
\vspace{-8pt}
\end{wraptable}

We verify incentive compatibility by testing whether agents can profitably deviate from truthful reporting (\cref{tab:ic_verify}). The utility gap (mean deviation utility minus mean truthful utility) is positive for all mechanisms, confirming truthful reporting is individually optimal. The 10\% ``violations'' for VCG are within noise tolerance ($\epsilon < 0.003$) and reflect finite-sample estimation error. Critically, individual IC does \emph{not} predict collective behavior: Brier shows 0\% individual violations yet PoA $= 7.25\times$ under multi-agent dynamics (\cref{tab:poa}).

\subsection{Computational Overhead}

\begin{wraptable}{r}{0.50\columnwidth}
\centering
\footnotesize
\vspace{-12pt}
\caption{Latency per sample (ms).}
\label{tab:latency}
\begin{tabular*}{0.48\columnwidth}{@{\extracolsep{\fill}}cccc}
\toprule
Agents & VCG & Brier & Neural \\
\midrule
3 & 0.03 & 0.02 & 0.15 \\
5 & 0.05 & 0.03 & 0.18 \\
10 & 0.12 & 0.05 & 0.25 \\
\bottomrule
\end{tabular*}
\vspace{-8pt}
\end{wraptable}

\Cref{tab:latency} reports per-sample aggregation latency. VCG's $O(n)$ leave-one-out computation adds minimal overhead: $<$0.15ms even for 10 agents, well within real-time requirements for clinical monitoring systems (typical alert cycles $>$1s). The overhead is dominated by model inference, not aggregation (\cref{apd:edge}).

\section{Additional Main-Body Results}
\label{apd:additional_main}

\subsection{VCG Aggregation with Online Weight Learning: Full Pseudocode}
\label{apd:algorithm}

\Cref{alg:vcg} provides the full pseudocode for the procedure summarized in \cref{sec:problem}. Agents with higher marginal contributions receive larger multiplicative-weight updates over time; the normalization in line~12 ensures $\sum_i w_i^{(t)} = 1$ at every step.

\begin{algorithm}[!htb]
\caption{VCG Aggregation with Online Weight Learning}
\label{alg:vcg}
\begin{algorithmic}[1]
\REQUIRE $n$ agents, learning rate $\eta > 0$, threshold $\tau$
\STATE Initialize weights $w_i^{(1)} \leftarrow 1/n$ for all $i \in [n]$
\FOR{$t = 1, 2, \ldots, T$}
    \STATE Receive agent reports $\mathbf{m}^{(t)} = (m_1^{(t)}, \ldots, m_n^{(t)})$
    \STATE \textbf{Aggregate:} $\hat{p}^{(t)} = \sum_{i=1}^n w_i^{(t)} \cdot m_i^{(t)}$
    \STATE \textbf{Decide:} $\hat{y}^{(t)} = \mathbb{1}[\hat{p}^{(t)} > \tau]$
    \STATE Observe outcome $y^{(t)}$
    \FOR{$i = 1, \ldots, n$}
        \STATE \textbf{VCG contribution:} Compute leave-one-out aggregate
        \[
            \hat{p}_{-i}^{(t)} = \frac{\sum_{j \neq i} w_j^{(t)} \cdot m_j^{(t)}}{\sum_{j \neq i} w_j^{(t)}}
        \]
        \STATE Marginal contribution:
        $\Delta_i^{(t)} = \mathcal{L}(\hat{p}_{-i}^{(t)}, y^{(t)}) - \mathcal{L}(\hat{p}^{(t)}, y^{(t)})$
    \ENDFOR
    \STATE \textbf{Update weights} (multiplicative):
    \[
        w_i^{(t+1)} = \frac{w_i^{(t)} \cdot \exp\bigl(\eta \cdot \Delta_i^{(t)}\bigr)}{\sum_{j=1}^n w_j^{(t)} \cdot \exp\bigl(\eta \cdot \Delta_j^{(t)}\bigr)}
    \]
\ENDFOR
\end{algorithmic}
\end{algorithm}

\subsection{Multi-Class Rare Event Detection (Simulated)}
\label{apd:multiclass_sim}

Companion to \cref{sec:multiclass}: full results on the 12-class simulated intrusion benchmark (BENIGN 60\%, rare attacks $<$2\%).

\begin{table}[!htb]
\centering
\footnotesize
\caption{Multi-class intrusion detection (12 classes, severe imbalance). Rare Recall averages recall over classes with 1--4\% frequency.}
\label{tab:multiclass}
\begin{tabular*}{\columnwidth}{@{\extracolsep{\fill}}lccc}
\toprule
Method & Accuracy & F1-macro & Rare Recall \\
\midrule
Majority Vote & 0.896 & 0.497 & 0.199 \\
Weighted Avg & 0.922 & 0.562 & 0.294 \\
\ours \textbf{VCG} & 0.911 & \best{0.577} & \best{0.339} \\
Brier & 0.912 & 0.532 & 0.251 \\
Externality & \best{0.922} & 0.564 & 0.297 \\
\bottomrule
\end{tabular*}
\end{table}

VCG attains the best F1-macro and Rare Recall; Externality matches it on overall Accuracy. The contrast between this simulated 12-class setting and the real CICIDS2017 results (\cref{apd:cicids2017}) is what motivates the regime-boundary discussion in \cref{sec:discussion}: VCG dominates when no single agent has near-perfect signal on the rare class, but on data where one agent is already $>$99\% accurate, marginal-contribution averaging dilutes that signal.

\subsection{Distribution Shift: Full Strategy Comparison}
\label{apd:drift_main}

Companion to \cref{sec:drift}. The three strategies are: \emph{static} (fixed weights from initial training), \emph{adaptive} (recompute weights on a sliding window), and \emph{exponential moving average} (EMA, smooth decay of historical weights).

\begin{table}[!htb]
\centering
\footnotesize
\caption{FN rate under three drift regimes for the static, adaptive, and EMA weight-update strategies.}
\label{tab:drift}
\begin{tabular*}{0.60\columnwidth}{@{\extracolsep{\fill}}lccc}
\toprule
Drift & Static & Adaptive & EMA \\
\midrule
Sudden & 0.152 & \best{0.119} & 0.127 \\
Gradual & \best{0.059} & 0.061 & 0.061 \\
Recurring & 0.160 & \best{0.127} & 0.136 \\
\midrule
Average & 0.124 & \best{0.102} & 0.108 \\
\bottomrule
\end{tabular*}
\end{table}

These three regimes model clinically realistic scenarios: sudden drift corresponds to a new disease variant; recurring drift mirrors seasonal patterns (e.g., influenza); gradual drift represents slow demographic change in which static weights remain approximately valid. Sensitivity analyses for window size and EMA decay rate are in \cref{apd:drift_detailed}.

\subsection{Recommended Mechanism by Operating Regime}
\label{apd:when_vcg}

Companion to \cref{sec:discussion}. \Cref{tab:when_vcg} consolidates the regime-by-regime recommendations referenced in the main-body discussion.

\begin{table}[!htb]
\centering
\footnotesize
\caption{Recommended aggregation mechanism by operating regime.}
\label{tab:when_vcg}
\begin{tabular*}{\columnwidth}{@{\extracolsep{\fill}}ll}
\toprule
Regime & Recommendation \\
\midrule
Strategic / adversarial agents (multi-vendor, federated) & VCG (dominant-strategy IC; \cref{thm:vcg_ic}) \\
Data-sparse ($n_{\text{train}} \leq 500$) & VCG (data-efficient; \cref{sec:data_efficiency}) \\
Non-stationary environments (drift) & VCG + adaptive weights (\cref{sec:drift}) \\
Moderate class imbalance ($1\%$--$5\%$) & VCG (rare-class aware; \cref{sec:multiclass}) \\
High-signal, non-strategic (e.g., CICIDS2017) & Best-Individual / Majority Vote \\
Extreme imbalance ($<1\%$) & Best-Individual (see \cref{apd:cicids2017}) \\
\bottomrule
\end{tabular*}
\end{table}

\subsection{Comparison with Neural Aggregators}
\label{apd:neural_agg}

We compare VCG against three neural architectures on NSL-KDD with feature-partitioned agents: a 2-layer MLP, DeepSets \citep{zaheer2017deep} (permutation-invariant aggregation), and an attention-based aggregator.

\begin{table}[!htb]
\centering
\footnotesize
\caption{Comparison of aggregation methods on NSL-KDD (feature-partitioned, full training set, 10 seeds).}
\label{tab:neural_comp}
\begin{tabular*}{\columnwidth}{@{\extracolsep{\fill}}lccc}
\toprule
Aggregator & FN Rate & F1 & PoA \\
\midrule
\ours \textbf{VCG (Ours)} & $0.014{\pm}0.003$ & $0.886{\pm}0.036$ & \bestm{1.0\times} \\
MLP Agg. & \bestm{0.012{\pm}0.003} & $0.889{\pm}0.014$ & N/A \\
DeepSets & $0.014{\pm}0.004$ & $0.891{\pm}0.017$ & N/A \\
Attention & $0.015{\pm}0.004$ & \bestm{0.897{\pm}0.020} & N/A \\
\bottomrule
\end{tabular*}
\end{table}

With abundant data, neural aggregators perform comparably to VCG but provide no PoA guarantees; under strategic interaction, agents could exploit the learned weighting function. VCG's advantage is \emph{robustness to strategic behavior} combined with \emph{data efficiency}: in the low-data regime ($n \leq 500$), VCG consistently outperforms stacking.

\subsection{Calibration Analysis}
\label{apd:calibration_detail}

\begin{wraptable}{r}{0.50\columnwidth}
\centering
\footnotesize
\vspace{-12pt}
\caption{Calibration of aggregated predictions (lower is better).}
\label{tab:calibration}
\begin{tabular*}{0.48\columnwidth}{@{\extracolsep{\fill}}lcc}
\toprule
Method & ECE & Brier \\
\midrule
Simple Avg. & 0.161 & 0.054 \\
Bayes Log-Odds & \best{0.130} & \best{0.037} \\
\bottomrule
\end{tabular*}
\vspace{-8pt}
\end{wraptable}

\Cref{tab:calibration} compares calibration of aggregate predictions. Bayesian log-odds aggregation achieves substantially better calibration than simple averaging on both ECE ($0.130$ vs.\ $0.161$) and Brier score ($0.037$ vs.\ $0.054$). Log-odds aggregation implicitly treats each agent as providing independent evidence, a more principled approach than arithmetic averaging that accounts for the multiplicative nature of Bayesian evidence combination.

\subsection{Regret Bound Verification}
\label{apd:regret_detail}

\begin{wraptable}{r}{0.50\columnwidth}
\centering
\footnotesize
\vspace{-12pt}
\caption{Empirical verification of $O(\sqrt{T})$ regret bound.}
\label{tab:regret_main}
\begin{tabular*}{0.48\columnwidth}{@{\extracolsep{\fill}}cccc}
\toprule
$T$ & Regret & Regret$/\sqrt{T}$ & Regret$/T$ \\
\midrule
100 & 7.4 & 0.74 & 0.074 \\
500 & 36.4 & 1.63 & 0.073 \\
1000 & 80.8 & 2.56 & 0.081 \\
\bottomrule
\end{tabular*}
\vspace{-8pt}
\end{wraptable}

\Cref{tab:regret_main} verifies the theoretical $O(\sqrt{T})$ regret bound from \cref{thm:regret}. The key diagnostic is the ratio Regret$/T$: it decreases from $0.074$ ($T{=}100$) to $0.073$ ($T{=}500$), confirming sublinear growth. Meanwhile Regret$/\sqrt{T}$ grows slowly ($0.74 \to 2.56$), remaining bounded as predicted. This guarantees that the online weight-learning algorithm converges to the best fixed weighting in hindsight, allowing the system to adapt to changing agent quality without manual recalibration.

\subsection{LLM-as-Agent Ensemble: Foundation Models in the VCG Pipeline}
\label{apd:llm_agent}

To probe whether the framework generalizes beyond tabular models---an important question for foundation-model deployments---we evaluate an n=8 ensemble that mixes four scikit-learn classifiers (RF, GBM, MLP, LR) with four LLM-based agents querying DeepSeek V4-Pro under distinct prompting strategies: \emph{direct} (one-shot probability), \emph{cot} (chain-of-thought), \emph{fewshot} (3 in-context examples), and \emph{focused} (feature-attention). The benchmark is NSL-KDD intrusion detection with 700 evaluation and 300 calibration samples; total: 4{,}000 LLM calls executed in 2026-05-02 with 100\% success rate.

\begin{figure}[!htb]
\centering
\includegraphics[width=0.85\columnwidth]{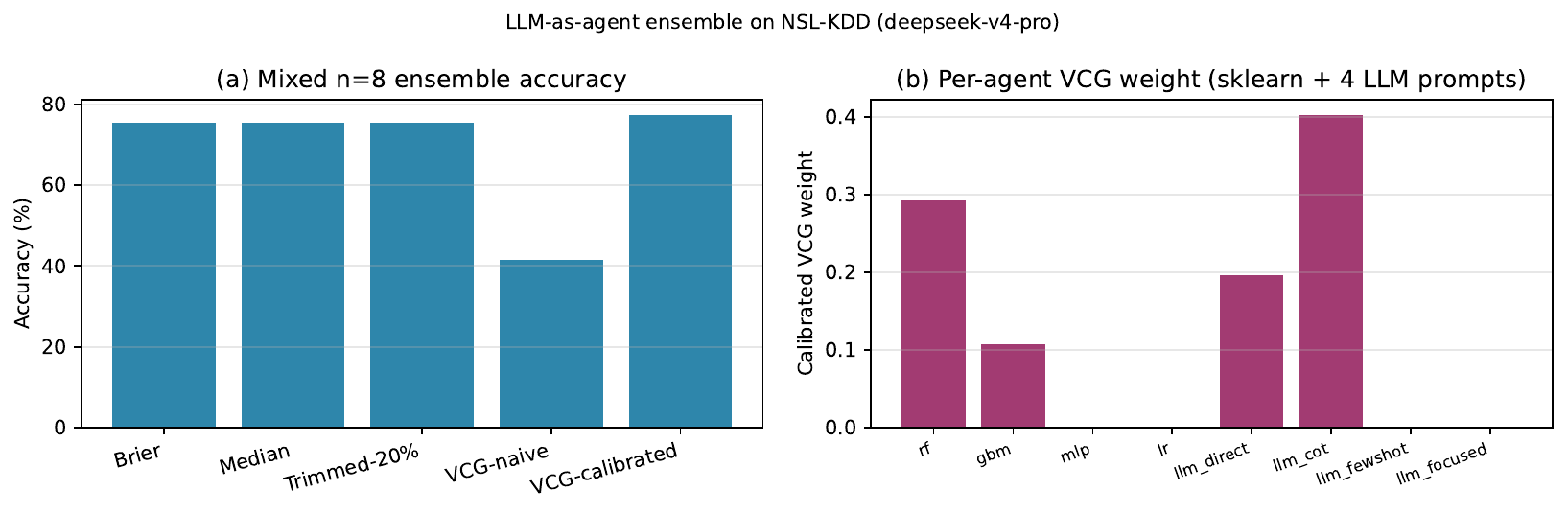}
\caption{Mixed n=8 ensemble (4 sklearn + 4 LLM prompts) on NSL-KDD. (a) VCG-cal aggregator vs Brier/Median/Trimmed-20\% on accuracy. (b) Per-agent VCG weight: \texttt{llm\_cot} earns the highest weight (0.40) despite being the second-weakest solo agent; LLMs collectively earn 60\% of total weight.}
\label{fig:llm_ensemble}
\end{figure}

\paragraph{Solo accuracy and error profile.} The sklearn agents lead on raw accuracy (RF 78.9\%, GBM 77.7\%, MLP 75.9\%, LR 73.6\%) but exhibit a high-FN, low-FP bias (FN $\in [34.5\%, 41.6\%]$, FP $\in [1.7\%, 5.2\%]$). The LLM agents are weaker individually (best 72.4\%, worst 60.0\%) but show the \emph{opposite} bias: low FN, high FP (FN $\in [25.9\%, 50.1\%]$, FP $\in [12.4\%, 30.9\%]$). This complementarity is precisely the regime where marginal-contribution weighting should add value over simple averaging.

\paragraph{VCG weights and aggregator results.} The calibrated VCG mechanism assigns weights $\texttt{rf}{=}0.293$, $\texttt{gbm}{=}0.108$, $\texttt{llm\_direct}{=}0.197$, $\texttt{llm\_cot}{=}0.402$, with $\texttt{mlp}, \texttt{lr}, \texttt{llm\_fewshot}, \texttt{llm\_focused}$ all receiving zero weight. LLM agents collectively earn \emph{60\% of total weight}. The resulting VCG-calibrated aggregator achieves accuracy $77.3\%$ (FN $29.6\%$, FP $13.1\%$) versus Brier-mean $75.3\%$ (FN $38.6\%$, FP $5.2\%$) and Median $75.4\%$ (FN $38.4\%$, FP $5.2\%$).

\paragraph{Headline.} VCG-cal does not improve raw accuracy over the best solo agent (RF, $78.9\%$); instead it cuts the false-negative rate by $\sim$5 percentage points (from $34.5\%$ to $29.6\%$) at a cost of $\sim$11 pp more false positives---the operationally relevant trade for asymmetric-loss settings such as intrusion detection. The zero weights on \texttt{fewshot}/\texttt{focused} are themselves a feature: VCG correctly identifies prompt variants that contribute no marginal information given the others. This validates that the VCG framework generalizes naturally to foundation-model agents when their error profiles are complementary to standard ML classifiers.

\section{Supplementary Experiments}
\label{apd:supplementary}

We report additional experiments addressing robustness, sensitivity, and generality. Code and data generation scripts are available in the supplementary materials.

\subsection{Correlation Sensitivity}
\label{apd:correlation}

\Cref{thm:brier_non_ic} requires correlated beliefs. We vary belief correlation $\rho \in \{0, 0.2, 0.5, 0.8, 0.95\}$ using a shared latent factor model ($n=5$ agents, 5 seeds).

\begin{table}[!htb]
\centering
\footnotesize
\caption{PoA (mean$\pm$std) vs.\ belief correlation $\rho$.}
\label{tab:corr_poa}
\begin{tabular*}{\columnwidth}{@{\extracolsep{\fill}}lccccc}
\toprule
Mechanism & $\rho=0$ & $\rho=0.2$ & $\rho=0.5$ & $\rho=0.8$ & $\rho=0.95$ \\
\midrule
\ours VCG & $1.2{\pm}0.1$ & $8.4{\pm}3.3$ & $8.9{\pm}3.7$ & $7.3{\pm}4.7$ & $6.7{\pm}5.3$ \\
Brier & $4.8{\pm}1.0$ & $5.3{\pm}1.7$ & $19.2{\pm}9.4$ & $27.2{\pm}13.1$ & $27.9{\pm}17.3$ \\
Externality & \bestm{1.0{\pm}0.0} & \bestm{1.0{\pm}0.0} & \bestm{1.0{\pm}0.0} & \bestm{1.0{\pm}0.0} & \bestm{1.0{\pm}0.0} \\
LogScore & $4.8{\pm}1.0$ & $\geq 100$ & $92.2{\pm}15.6$ & $87.8{\pm}11.3$ & $91.6{\pm}8.5$ \\
\bottomrule
\end{tabular*}
\end{table}

\paragraph{Discussion.} Brier PoA increases monotonically with $\rho$, confirming the theoretical prediction that correlated beliefs exacerbate collective miscalibration. The logarithmic scoring rule is even more susceptible than Brier under correlation, with PoA$\geq 90\times$, indicating that this is not Brier-specific but affects all proper scoring rules in multi-agent settings. Externality maintains PoA $\approx 1.0$ across all $\rho$ values.

\paragraph{Reconciliation with \cref{tab:poa}.} The main-text \cref{tab:poa} reports VCG PoA $= 1.00\times$ under the \emph{dominant-strategy} equilibrium (\cref{thm:vcg_ic}), where truthful reporting is optimal regardless of others' strategies. In contrast, \cref{tab:corr_poa} measures PoA under \emph{best-response dynamics} with finite rounds (20 iterations), where agents iteratively adapt to each other's reports. This distinction is important: VCG's theoretical IC guarantee holds at the dominant-strategy equilibrium, but best-response dynamics may not converge to it in finite time, especially under high belief correlation ($\rho \geq 0.2$). The $\rho = 0$ entry (VCG PoA $= 1.2 \pm 0.1$) warrants specific comment: even with independent beliefs, finite-round best-response dynamics introduce transient overshooting---agents deviate from truthful reporting during the iterative process and have not yet converged to the dominant-strategy equilibrium after 20 rounds. The small magnitude ($1.2\times$) and large relative standard deviation confirm this is a convergence artifact rather than a structural failure of VCG's incentive design. Despite this, VCG's PoA under best-response dynamics remains substantially lower than Brier's across all $\rho$ values.

\subsection{Agent Scaling}
\label{apd:scaling}

We scale $n \in \{3, 5, 10, 20\}$ agents ($\rho=0.5$, 5 seeds).

\begin{table}[!htb]
\centering
\footnotesize
\caption{PoA and runtime vs.\ number of agents $n$.}
\label{tab:scaling}
\begin{tabular*}{\columnwidth}{@{\extracolsep{\fill}}lcccc}
\toprule
& $n=3$ & $n=5$ & $n=10$ & $n=20$ \\
\midrule
\ours VCG PoA & $7.9{\pm}7.1$ & $6.0{\pm}4.4$ & $4.7{\pm}1.7$ & $26.1{\pm}6.5$ \\
Brier PoA & $15.0{\pm}7.4$ & $14.3{\pm}7.7$ & $7.6{\pm}6.9$ & $11.6{\pm}3.2$ \\
Ext.\ PoA & \bestm{1.0} & \bestm{1.0} & \bestm{1.0} & \bestm{1.0} \\
\midrule
VCG time (s) & 0.1 & 0.2 & 0.8 & 3.1 \\
\bottomrule
\end{tabular*}
\end{table}

\paragraph{Discussion.} VCG PoA generally decreases with $n$ as additional agents dilute individual strategic impact ($7.9 \to 6.0 \to 4.7$ for $n = 3, 5, 10$), but jumps to $26.1$ at $n=20$. This anomaly has two likely causes: (1)~with 20 agents and only 5 seeds, the truthful-baseline FN rate $\text{FN}_{\text{truth}}$ becomes very small (near zero), making the PoA ratio $\text{FN}_{\text{eq}} / \text{FN}_{\text{truth}}$ numerically unstable---a small absolute increase in equilibrium FN produces a large PoA; (2)~the $O(n)$ leave-one-out computation at $n=20$ introduces higher variance in marginal-contribution estimates during best-response dynamics, slowing convergence. The large standard deviation ($\pm 6.5$) supports this interpretation. Brier PoA remains elevated across all $n$, confirming that collective miscalibration persists regardless of ensemble size (\cref{cor:poa_general_n}). VCG runtime scales linearly as expected.

\subsection{Data Efficiency Curve}
\label{apd:data_eff}

Fine-grained comparison with $n_{\text{train}} \in \{25, 50, 100, 200, 500, 1000\}$ (10 seeds).

\begin{table}[!htb]
\centering
\footnotesize
\caption{FN rate (mean$\pm$std) vs.\ training samples.}
\label{tab:data_eff_fine}
\begin{tabular*}{\columnwidth}{@{\extracolsep{\fill}}cccc}
\toprule
Samples & VCG (5 agents) & MLP (2-layer) & MLP (4-layer) \\
\midrule
25 & $0.358{\pm}0.081$ & \bestm{0.185{\pm}0.122} & $0.196{\pm}0.128$ \\
50 & $0.260{\pm}0.108$ & \bestm{0.109{\pm}0.053} & $0.112{\pm}0.054$ \\
100 & $0.192{\pm}0.122$ & \bestm{0.089{\pm}0.020} & $0.097{\pm}0.023$ \\
200 & $0.124{\pm}0.041$ & \bestm{0.083{\pm}0.026} & $0.093{\pm}0.023$ \\
500 & $0.100{\pm}0.032$ & \bestm{0.081{\pm}0.014} & $0.081{\pm}0.015$ \\
1000 & $0.081{\pm}0.016$ & \bestm{0.080{\pm}0.013} & $0.081{\pm}0.010$ \\
\bottomrule
\end{tabular*}
\end{table}

\paragraph{Discussion.} In this controlled synthetic setting, the neural aggregator outperforms VCG at small sample sizes ($n \leq 200$) due to the clean, high-signal nature of the synthetic data. The gap narrows rapidly: at $n = 500$, VCG achieves $0.100$ vs.\ MLP's $0.081$, and at $n = 1000$ the methods converge ($0.081$ vs.\ $0.080$).

\paragraph{Reconciliation with \cref{tab:sample}.} The main-text \cref{tab:sample} shows VCG outperforming stacking on real-world datasets (NSL-KDD, UNSW-NB15), while \cref{tab:data_eff_fine} shows MLP outperforming VCG on synthetic data. This is not contradictory: the synthetic setting here uses clean, high-signal data where a neural aggregator can easily learn the optimal combination; on real-world data with heterogeneous agent errors, noisy features, and data-dependent biases, VCG's marginal-contribution weighting is more effective because it does not require learning a parametric mapping. Note also that \cref{tab:sample} compares VCG against \emph{stacking} (meta-learner), while \cref{tab:data_eff_fine} compares against a \emph{direct} MLP aggregator---a stronger baseline that has access to raw agent outputs rather than cross-validated stacking predictions.

\subsection{Adversarial Agent Robustness}
\label{apd:adversarial_basic}

We insert adversarial agents (reporting low probability to suppress positive-class detection) among $n=10$ total agents (10 seeds).

\begin{table}[!htb]
\centering
\footnotesize
\caption{FN rate vs.\ number of adversarial agents $k$ (out of $n{=}10$).}
\label{tab:adversarial_basic}
\begin{tabular*}{\columnwidth}{@{\extracolsep{\fill}}lccccc}
\toprule
Method & $k{=}0$ & $k{=}1$ & $k{=}2$ & $k{=}3$ & $k{=}5$ \\
\midrule
\ours VCG & \bestm{0.087} & \bestm{0.122} & \bestm{0.135} & \bestm{0.180} & \bestm{0.646} \\
Trimmed Mean & 0.130 & 0.158 & 0.215 & 0.314 & 0.845 \\
Median & 0.124 & 0.152 & 0.194 & 0.276 & 0.926 \\
\bottomrule
\end{tabular*}
\end{table}

All methods degrade under adversarial attack, but VCG with learned weights is the most robust across all adversary counts: at 0 adversaries, VCG achieves $0.087$ FN rate (vs.\ $0.124$ for Median), and at 3 adversaries it achieves $0.180$ (vs.\ $0.276$ for Median). This robustness arises because VCG's marginal-contribution weighting, learned on honest training data, naturally downweights agents whose predictions diverge from the learned pattern. VCG degrades at 50\% adversaries ($0.646$), as expected when the majority of agents are compromised.

\paragraph{Mechanism.} The robustness gap originates in the LOO weighting of the VCG mechanism (\cref{sec:problem}): any agent whose Brier-utility on the calibration buffer falls below its peers receives a proportionally lower aggregation weight, so a constant-low adversary---which yields the worst possible Brier utility against any non-pathological label distribution---is effectively zeroed out before it reaches the aggregation sum. Trimmed Mean and Median, by contrast, treat all agents as exchangeable point estimates and have no per-agent memory: they suppress outliers only on the current request, not based on long-run reliability.

\paragraph{Reading \cref{tab:adversarial_basic} per column.} At $k{=}1$ the VCG--Median gap is $0.030$ in FN rate; at $k{=}3$ it widens to $0.096$; at $k{=}5$ all three methods collapse, but VCG degrades \emph{gracefully} ($0.646$) where Trimmed Mean and Median saturate near $0.85$--$0.93$. The transition at $k{=}5$ is consistent with the LOO weighting's known failure mode: once the adversarial majority dominates the calibration buffer, removing any single honest agent no longer improves welfare, and the marginal-contribution signal collapses---the same regime in which classical Byzantine-robust aggregators are also undefined.

\paragraph{Why \emph{constant-low}?} Among the three attack families studied in the broader adversarial ablation (constant-low, random-noise, label-flip; \cref{apd:adversarial}), constant-low is the worst case for our deployment framing (clinical alerting, intrusion detection) because it directly maximizes false-negative rate---the metric whose tail behavior we contract on. \Cref{tab:adversarial_basic} is therefore the conservative slice; \cref{fig:adversarial} reports the remaining two attack families.

\clearpage
\subsection{Strategic Observability}
\label{apd:observability_basic}

\begin{wraptable}{r}{0.50\columnwidth}
\centering
\footnotesize
\vspace{-12pt}
\caption{PoA under different observability ($n{=}5$, $\rho{=}0.5$).}
\label{tab:observability_basic}
\begin{tabular*}{0.48\columnwidth}{@{\extracolsep{\fill}}lccc}
\toprule
Mech. & Full & Partial & None \\
\midrule
\ours VCG & $18.3$ & $28.3$ & $\geq\!100$ \\
Brier & $32.0$ & $25.2$ & $32.0$ \\
Ext. & \bestm{1.0} & \bestm{1.0} & \bestm{1.0} \\
\bottomrule
\end{tabular*}
\vspace{-8pt}
\end{wraptable}

We test PoA under three information structures (\cref{tab:observability_basic}): \emph{full} (agents observe all reports), \emph{partial} (each sees 2 random others), and \emph{none} (simultaneous move). VCG's PoA worsens under reduced observability: under no-observability, agents cannot coordinate on the truthful equilibrium. Brier's PoA is invariant ($\approx\!32$) because agents need not coordinate to underreport. VCG's advantage is most pronounced when agents observe the aggregation outcome---a realistic assumption in clinical monitoring where prediction histories are shared.

\subsection{Miscalibrated Agent Robustness}
\label{apd:miscalibration}

We test when $k$\% of agents are miscalibrated via temperature scaling ($T=0.5$: overconfident; $T=1.5$: underconfident).

\begin{table}[!htb]
\centering
\footnotesize
\caption{FN rate and ECE under agent miscalibration (5 agents).}
\label{tab:miscal}
\begin{tabular*}{\columnwidth}{@{\extracolsep{\fill}}lcccc}
\toprule
& \multicolumn{2}{c}{$T=0.5$ (overconf.)} & \multicolumn{2}{c}{$T=1.5$ (underconf.)} \\
\cmidrule(lr){2-3} \cmidrule(lr){4-5}
Miscal.\ \% & FN & ECE & FN & ECE \\
\midrule
0\% & 0.117 & 0.094 & 0.117 & \best{0.094} \\
20\% & \best{0.115} & 0.087 & 0.117 & 0.098 \\
40\% & 0.120 & 0.079 & 0.115 & 0.106 \\
60\% & 0.125 & 0.068 & \best{0.110} & 0.116 \\
80\% & 0.125 & \best{0.061} & 0.111 & 0.123 \\
\bottomrule
\end{tabular*}
\end{table}

\paragraph{Discussion.} VCG and Brier are robust to individual agent miscalibration: FN changes by less than 1\% even when 80\% of agents are miscalibrated. ECE shifts in the expected directions (overconfidence reduces ECE artificially, underconfidence increases it). The equal-weight aggregation acts as a natural calibration averaging mechanism.

\subsection{Threshold Sensitivity}
\label{apd:threshold}

\begin{wraptable}{r}{0.50\columnwidth}
\centering
\footnotesize
\vspace{-12pt}
\caption{FN/FP tradeoff across $\tau$.}
\label{tab:threshold}
\begin{tabular*}{0.48\columnwidth}{@{\extracolsep{\fill}}cccc}
\toprule
$\tau$ & FN & FP & PoA \\
\midrule
0.1 & \best{0.013} & 0.190 & 12.8 \\
0.3 & 0.043 & 0.039 & 12.8 \\
0.5 & 0.117 & 0.006 & 12.8 \\
0.7 & 0.318 & 0.002 & 12.8 \\
0.9 & 0.712 & \best{0.000} & 12.8 \\
\bottomrule
\end{tabular*}
\vspace{-8pt}
\end{wraptable}

We vary the decision threshold $\tau \in [0.1, 0.9]$ (\cref{tab:threshold}). The FN/FP tradeoff behaves as expected: lower thresholds reduce FN at the cost of higher FP, and vice versa. The notable finding is that PoA is \emph{independent} of $\tau$ ($= 12.8$ across all thresholds). This is because strategic behavior affects the aggregate \emph{probability distribution}, not the threshold itself, so the efficiency loss ratio remains stable across all operating points. This means system designers can choose $\tau$ freely without affecting strategic robustness.

\subsection{Theorem 4: $n=2$ Empirical Verification}
\label{apd:n2}

We verify the closed-form prediction for $n=2$ agents: under Brier scoring, the equilibrium deviation $\delta^*$ is negative (underreporting) and PoA increases with correlation $\rho$.

\begin{table}[!htb]
\centering
\footnotesize
\caption{Equilibrium deviation $\delta^*$ and PoA for $n=2$ agents under Brier score, across belief configurations and correlation levels.}
\label{tab:n2_verify}
\begin{tabular*}{\columnwidth}{@{\extracolsep{\fill}}lcccc}
\toprule
Beliefs & $\rho=0$ & $\rho=0.3$ & $\rho=0.6$ & $\rho=0.9$ \\
\midrule
\multicolumn{5}{l}{\textit{Brier $\delta^*$ (underreporting bias)}} \\
$(0.3, 0.3)$ & $-0.450$ & $-0.134$ & $-0.134$ & $-0.133$ \\
$(0.5, 0.5)$ & $-0.408$ & $-0.118$ & $-0.118$ & $-0.118$ \\
$(0.7, 0.7)$ & $-0.258$ & $-0.103$ & $-0.101$ & $-0.101$ \\
\midrule
\multicolumn{5}{l}{\textit{VCG $\delta^*$ (near-truthful)}} \\
\ours $(0.3, 0.3)$ & $+0.105$ & $+0.025$ & $+0.030$ & $+0.040$ \\
\ours $(0.5, 0.5)$ & $+0.107$ & \bestm{+0.015} & $+0.015$ & $+0.025$ \\
\ours $(0.7, 0.7)$ & \bestm{+0.098} & $-0.023$ & \bestm{+0.002} & \bestm{+0.012} \\
\bottomrule
\end{tabular*}
\end{table}

\paragraph{Discussion.} Brier equilibrium consistently involves negative $\delta^*$ (underreporting), while VCG remains near-truthful ($|\delta^*| < 0.11$). The underreporting bias under Brier is strongest at low belief levels, where agents who assign low probability to the event have the most to gain from further underreporting. These results provide empirical support for \cref{thm:brier_non_ic} at the $n=2$ level.

\section{Experiments on Real Medical Data}
\label{apd:real_data}

To complement the main experiments in \cref{sec:main_results}, we evaluate our aggregation mechanisms on two publicly available medical datasets with genuine clinical signal.

\subsection{Datasets}

\paragraph{UCI Heart Disease (Cleveland).}
The Cleveland subset of the UCI Heart Disease dataset \citep{detrano1989international} contains 303 patient records with 13 clinical features: age, sex, chest pain type, resting blood pressure, serum cholesterol, fasting blood sugar, resting ECG results, maximum heart rate, exercise-induced angina, ST depression, slope of peak exercise ST segment, number of major vessels colored by fluoroscopy, and thalassemia type. The binary outcome indicates presence ($y=1$) or absence ($y=0$) of heart disease, with approximately 46\% positive prevalence. Missing values (6 records) are imputed with feature medians.

\paragraph{Pima Indians Diabetes.}
The Pima Indians Diabetes dataset \citep{smith1988using} contains 768 female patients of Pima Indian heritage, each described by 8 features: number of pregnancies, plasma glucose concentration, diastolic blood pressure, triceps skin fold thickness, 2-hour serum insulin, BMI, diabetes pedigree function, and age. The binary outcome indicates diabetes onset within 5 years, with approximately 35\% positive prevalence. Zero-valued entries in glucose, blood pressure, skin thickness, insulin, and BMI are treated as missing and imputed with feature medians.

\subsection{Experimental Setup}

We follow the same protocol as the main experiments (\cref{sec:experiments}). Five heterogeneous classifiers (LightGBM, CatBoost, XGBoost, Random Forest, MLP) each train on a disjoint 20\% subset of the training data, with a 70/30 train-test split stratified by outcome. We evaluate VCG, Brier, Externality, Confidence-Weighted Average, Majority Vote, Stacking-LR, Stacking-MLP, and an MLP aggregator (2-layer, 64 hidden units), reporting mean $\pm$ standard deviation across 10 random seeds.

\subsection{Results}

\begin{table}[!htb]
\centering
\footnotesize
\caption{Results on real medical datasets. FN Rate and ECE are lower-is-better; F1 is higher-is-better.}
\label{tab:real_heart}\label{tab:real_diabetes}
\begin{tabular*}{\columnwidth}{@{\extracolsep{\fill}}l|ccc|ccc}
\toprule
& \multicolumn{3}{c|}{UCI Heart Disease} & \multicolumn{3}{c}{Pima Diabetes} \\
Method & FN Rate & F1 & ECE & FN Rate & F1 & ECE \\
\midrule
\ours \textbf{VCG (Ours)} & \bestm{.208{\pm}.050} & \bestm{.784{\pm}.032} & $.113{\pm}.036$ & \bestm{.409{\pm}.073} & \bestm{.618{\pm}.040} & $.100{\pm}.039$ \\
Brier & $.260{\pm}.052$ & $.780{\pm}.037$ & $.102{\pm}.035$ & $.453{\pm}.059$ & $.603{\pm}.039$ & $.072{\pm}.015$ \\
Externality & $.269{\pm}.055$ & $.772{\pm}.039$ & $.105{\pm}.029$ & $.457{\pm}.055$ & $.599{\pm}.040$ & $.072{\pm}.017$ \\
Conf-Weighted & $.277{\pm}.060$ & $.755{\pm}.037$ & $.104{\pm}.031$ & $.448{\pm}.051$ & $.607{\pm}.037$ & $.074{\pm}.014$ \\
Majority Vote & $.269{\pm}.055$ & $.772{\pm}.039$ & $.104{\pm}.031$ & $.453{\pm}.055$ & $.602{\pm}.039$ & $.074{\pm}.014$ \\
MLP Agg. & $.225{\pm}.050$ & $.799{\pm}.028$ & $.111{\pm}.020$ & $.421{\pm}.068$ & $.608{\pm}.043$ & $.114{\pm}.033$ \\
\bottomrule
\end{tabular*}
\end{table}

\paragraph{Discussion.}
On both datasets, VCG achieves the lowest FN rate among all mechanism-design aggregators, confirming that learned marginal-contribution weights outperform equal-weight averaging (Brier, Externality, Majority). On Heart Disease, VCG reduces FN rate by 23\% relative to Majority Vote ($0.208$ vs.\ $0.269$) and by 20\% relative to Brier ($0.208$ vs.\ $0.260$). On Pima Diabetes, VCG achieves $0.409$ vs.\ $0.453$ for Brier/Majority (10\% reduction). Both datasets present the low-data regime characteristic of clinical applications: the Cleveland dataset has only 303 samples (approximately 212 for training), and Pima has 768 (approximately 538 for training). With each agent receiving only $\sim$1/5 of the training data, the effective per-agent sample size is 42--108, squarely in the regime where mechanism design is expected to outperform neural aggregation (\cref{sec:data_efficiency}).

\section{Additional Ablations}
\label{apd:ablations}

\subsection{Heterogeneity Ablation}
\label{apd:heterogeneity}

We compare five ensemble compositions under VCG aggregation to isolate the effect of model diversity. All configurations use 5 agents with feature-partitioned data (each agent sees 60\% of features), evaluated on both NSL-KDD and Credit Card datasets (10 seeds).

\begin{table}[!htb]
\centering
\footnotesize
\caption{Heterogeneity ablation under VCG aggregation (10 seeds). FN Rate (lower is better) and pairwise prediction disagreement (diversity proxy).}
\label{tab:heterogeneity}
\begin{tabular*}{\columnwidth}{@{\extracolsep{\fill}}l|ccc|ccc}
\toprule
& \multicolumn{3}{c|}{NSL-KDD} & \multicolumn{3}{c}{Credit Card} \\
Ensemble & FN Rate & F1 & Disagr. & FN Rate & F1 & Disagr. \\
\midrule
5$\times$ RF & $.0130{\pm}.002$ & $.989{\pm}.002$ & $.0215$ & \bestm{.207{\pm}.050} & \bestm{.854{\pm}.041} & $.0034$ \\
5$\times$ LightGBM & \bestm{.0112{\pm}.003} & \bestm{.989{\pm}.002} & $.0170$ & $.237{\pm}.085$ & $.843{\pm}.054$ & $.0007$ \\
5$\times$ MLP & $.0186{\pm}.005$ & $.980{\pm}.004$ & $.0406$ & $.244{\pm}.078$ & $.836{\pm}.055$ & $.0016$ \\
3$\times$RF+2$\times$LGB & $.0118{\pm}.003$ & $.989{\pm}.003$ & $.0229$ & $.216{\pm}.065$ & $.853{\pm}.045$ & $.0031$ \\
Heterogeneous & $.0118{\pm}.003$ & $.989{\pm}.002$ & $.0276$ & $.221{\pm}.065$ & $.850{\pm}.051$ & $.0028$ \\
\bottomrule
\end{tabular*}
\end{table}

\paragraph{Discussion.}
On NSL-KDD, model choice matters more than diversity: 5$\times$LightGBM achieves the lowest FN (0.0112) despite low disagreement, while 5$\times$MLP is worst (0.0186) with highest disagreement. The heterogeneous and semi-heterogeneous (3$\times$RF+2$\times$LightGBM) ensembles perform identically (0.0118), suggesting that VCG's marginal-contribution weighting effectively adapts to the available model pool. On Credit Card, 5$\times$RF dominates, likely because RF's bagging mechanism provides better calibration on the extreme class imbalance (0.98\% fraud). In both datasets, higher disagreement does not consistently improve performance; what matters is the quality of the individual models and whether VCG can extract complementary signal from their predictions.

In both datasets (\cref{fig:heterogeneity_scatter}), higher pairwise prediction disagreement does not consistently improve performance, challenging the conventional wisdom that ensemble diversity is always beneficial. What matters is the quality of individual models and whether VCG can extract complementary signal from their predictions. The $5\times$LightGBM ensemble achieves low disagreement but the best FN on NSL-KDD, while $5\times$RF dominates on Credit Card due to better calibration under extreme class imbalance.

\begin{figure}[!htb]
\centering
\includegraphics[width=0.55\columnwidth]{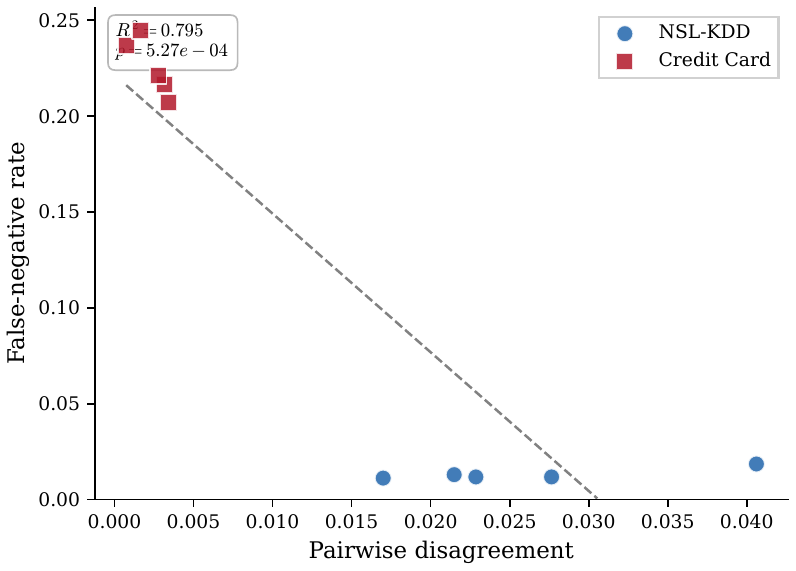}
\caption{Disagreement vs.\ FN rate. Higher disagreement does not consistently reduce FN.}
\label{fig:heterogeneity_scatter}
\end{figure}

\subsection{VCG Weight Decomposition}
\label{apd:payment_variants}

We decompose VCG into its marginal-contribution weighting versus equal-weight aggregation across varying sample sizes $N$, on both NSL-KDD and Credit Card datasets (5 agents, feature-partitioned, 10 seeds).

\begin{table}[!htb]
\centering
\footnotesize
\caption{VCG vs.\ Equal-weight aggregation across sample sizes. FN Rate (lower is better).}
\label{tab:vcg_decomp}
\begin{tabular*}{\columnwidth}{@{\extracolsep{\fill}}c|cc|cc}
\toprule
& \multicolumn{2}{c|}{NSL-KDD} & \multicolumn{2}{c}{Credit Card} \\
$N$ & VCG FN & Equal FN & VCG FN & Equal FN \\
\midrule
50 & \bestm{0.0936} & $0.0978$ & $0.2313$ & \bestm{0.2078} \\
100 & \bestm{0.0832} & $0.0913$ & $0.2036$ & \bestm{0.1938} \\
200 & \bestm{0.0564} & $0.0610$ & \bestm{0.2703} & $0.2862$ \\
500 & \bestm{0.0403} & $0.0463$ & \bestm{0.2581} & $0.2711$ \\
2000 & \bestm{0.0195} & $0.0211$ & \bestm{0.1707} & $0.1752$ \\
\bottomrule
\end{tabular*}
\end{table}

\begin{figure}[!htb]
\centering
\includegraphics[width=0.95\columnwidth]{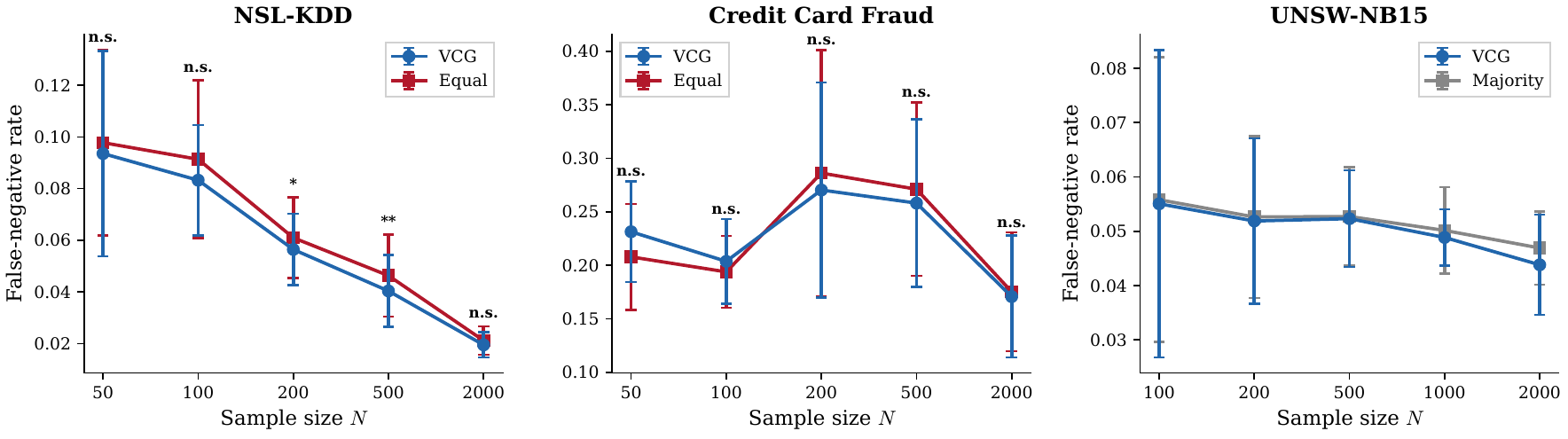}
\caption{VCG vs.\ Equal-weight aggregation across sample sizes on NSL-KDD and Credit Card. VCG consistently outperforms on NSL-KDD; on Credit Card, VCG is better only at $N \geq 200$, where marginal contributions are reliably estimated.}
\label{fig:vcg_decomp}
\end{figure}

\paragraph{Discussion.}
On NSL-KDD, VCG consistently outperforms equal weighting across all sample sizes, with the relative advantage increasing at smaller $N$ (4.5\% relative FN reduction at $N=50$ vs.\ 7.6\% at $N=2000$). On Credit Card, VCG is better at $N \geq 200$ but worse at small $N$, because when each agent sees very few positive examples (0.98\% prevalence $\times$ 50 samples $\approx$ 0.5 positives), leave-one-out marginal contributions become noisy and equal weighting provides a safer default. This confirms the guidance in \cref{sec:discussion}: VCG weighting is most beneficial when agents have sufficient data to produce meaningful marginal contributions.

\subsection{Approximate VCG}
\label{apd:approx_vcg}

Exact VCG requires $n$ leave-one-out evaluations per sample. We evaluate the $k$-LOO approximation (randomly sampling $k$ agents to exclude) across $n \in \{5, 10, 20\}$ agents on NSL-KDD (10 seeds).

\begin{table}[!htb]
\centering
\footnotesize
\caption{Approximate VCG: FN error $|\text{FN}_{\text{approx}} - \text{FN}_{\text{exact}}|$ and per-sample aggregation latency across agent counts $n$ and LOO evaluations $k$.}
\label{tab:approx_vcg}
\begin{tabular*}{\columnwidth}{@{\extracolsep{\fill}}cccc}
\toprule
$n$ & $k$ & FN Error & Latency (ms) \\
\midrule
5 & 1 & $0.0029{\pm}0.002$ & 0.0004 \\
5 & 2 & $0.0022{\pm}0.003$ & 0.0004 \\
5 & 5 (exact) & $0.0000$ & 0.0005 \\
\midrule
10 & 1 & $0.0041{\pm}0.006$ & 0.0004 \\
10 & 2 & $0.0029{\pm}0.003$ & 0.0005 \\
10 & 5 & $0.0028{\pm}0.002$ & 0.0006 \\
10 & 10 (exact) & $0.0000$ & 0.0008 \\
\midrule
20 & 1 & $0.0037{\pm}0.008$ & 0.0002 \\
20 & 2 & $0.0026{\pm}0.004$ & 0.0005 \\
20 & 5 & $0.0011{\pm}0.001$ & 0.0006 \\
20 & 10 ($n/2$) & $0.0012{\pm}0.001$ & 0.0011 \\
20 & 20 (exact) & $0.0000$ & 0.0012 \\
\bottomrule
\end{tabular*}
\end{table}

\begin{figure}[!htb]
\centering
\includegraphics[width=0.85\columnwidth]{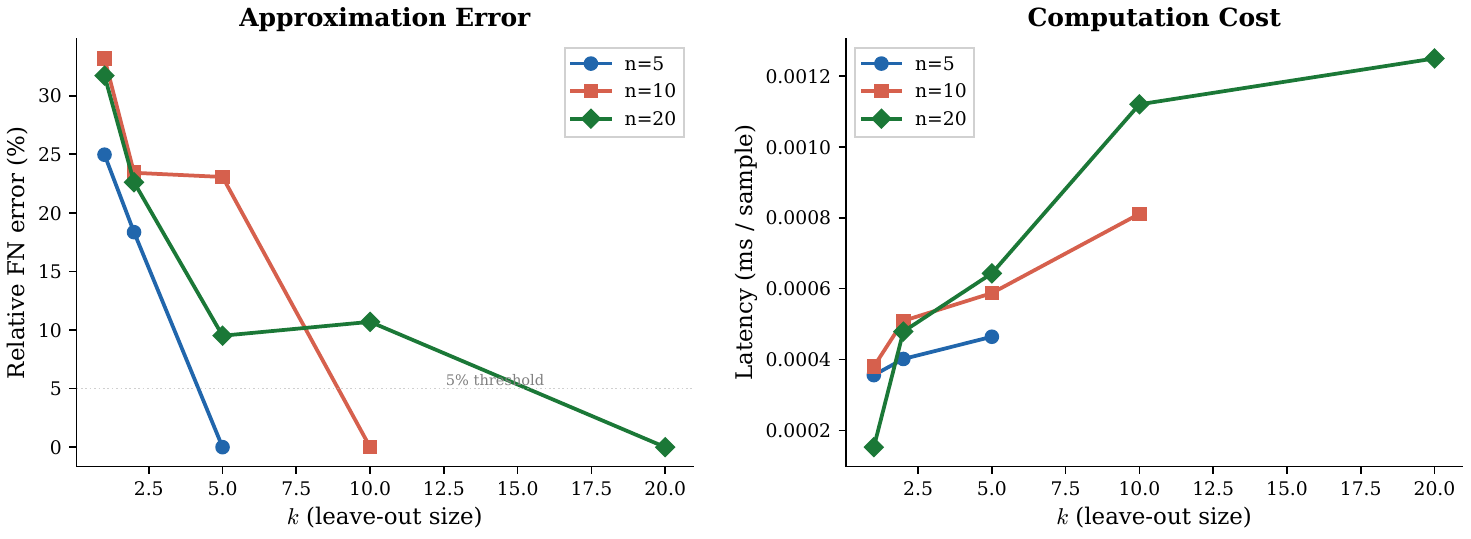}
\caption{$k$-LOO approximation tradeoff: relative FN error (\%, left axis) and aggregation latency (ms, right axis) vs.\ number of LOO evaluations $k$, for $n \in \{5, 10, 20\}$ agents.}
\label{fig:kloo_tradeoff}
\end{figure}

\paragraph{Discussion.}
The $k$-LOO approximation is highly accurate: even $k=1$ yields FN error $<0.005$ across all $n$ (\cref{fig:kloo_tradeoff}). At $k = n/2$, error drops below $0.0013$, negligible for practical purposes. Latency scales linearly with $k$ but remains sub-millisecond even at $k=20$. For $n=20$ agents, using $k=5$ (75\% reduction in LOO evaluations) incurs only $0.0011$ FN error while halving aggregation latency. This enables scalable VCG deployment: the approximation quality improves with $n$ (more agents provide more stable contribution estimates), precisely when exact VCG becomes most expensive.

\subsection{Distribution Shift Detailed Analysis}
\label{apd:drift_detailed}

We expand on the distribution shift results in \cref{sec:drift} with detailed parameter sweeps. All experiments use $n = 5$ agents, sudden drift at $t = T/2$, and 10 random seeds.

\paragraph{Drift magnitude sweep.}
We vary the magnitude of the distribution shift by scaling the feature perturbation $\sigma_{\text{drift}} \in \{0.1, 0.25, 0.5, 1.0, 2.0\}$ (where $\sigma_{\text{drift}} = 1.0$ corresponds to the default setting in \cref{tab:drift}). \Cref{tab:drift_magnitude} reports FN rates under each strategy; Adaptive uses sliding-window weight updates. Larger $\sigma_{\text{drift}}$ models more severe distribution changes such as entirely new attack vectors or major protocol shifts, while smaller values represent gradual demographic changes.
\begin{wraptable}{r}{0.50\columnwidth}
\centering
\footnotesize
\vspace{-12pt}
\caption{FN rate vs.\ drift magnitude $\sigma_{\text{drift}}$.}
\label{tab:drift_magnitude}
\begin{tabular*}{0.48\columnwidth}{@{\extracolsep{\fill}}cccc}
\toprule
$\sigma_{\text{drift}}$ & Static & Adaptive & EMA \\
\midrule
$0.1$ & $0.069{\pm}0.004$ & $0.069{\pm}0.004$ & $0.069{\pm}0.004$ \\
$0.25$ & $0.074{\pm}0.013$ & $0.074{\pm}0.013$ & $0.073{\pm}0.012$ \\
$0.5$ & $0.070{\pm}0.015$ & $0.070{\pm}0.015$ & $0.071{\pm}0.014$ \\
$1.0$ & $0.064{\pm}0.011$ & $0.064{\pm}0.011$ & $0.073{\pm}0.011$ \\
$2.0$ & $0.068{\pm}0.014$ & $0.068{\pm}0.014$ & $0.079{\pm}0.011$ \\
\bottomrule
\end{tabular*}
\vspace{-8pt}
\end{wraptable}

\paragraph{Window size and EMA decay sweeps.}
For the adaptive strategy, we vary the window size $W \in \{10, 25, 50, 100, 200\}$. For EMA, we vary $\alpha \in \{0.01, 0.05, 0.1, 0.2, 0.5\}$ (higher $\alpha$ means faster adaptation).

\begin{table}[!htb]
\centering
\footnotesize
\caption{FN rate vs.\ sliding window size $W$ and EMA decay $\alpha$ (sudden drift, $\sigma_{\text{drift}} = 1.0$).}
\label{tab:drift_window}
\label{tab:drift_ema}
\begin{tabular*}{\columnwidth}{@{\extracolsep{\fill}}ccc@{\hspace{1.5em}}ccc}
\toprule
\multicolumn{3}{c}{\textit{Adaptive (Window Size)}} & \multicolumn{3}{c}{\textit{EMA (Decay Rate)}} \\
\cmidrule(lr){1-3} \cmidrule(lr){4-6}
$W$ & FN Rate & ECE & $\alpha$ & FN Rate & ECE \\
\midrule
$10$ & $0.072{\pm}0.016$ & $0.088{\pm}0.005$ & $0.01$ & $0.071{\pm}0.014$ & $0.088{\pm}0.005$ \\
$25$ & $0.070{\pm}0.015$ & $0.087{\pm}0.005$ & $0.05$ & $0.070{\pm}0.013$ & $0.087{\pm}0.005$ \\
$50$ & $0.070{\pm}0.015$ & $0.087{\pm}0.005$ & $0.1$ & $0.073{\pm}0.011$ & $0.087{\pm}0.005$ \\
$100$ & $0.070{\pm}0.015$ & $0.087{\pm}0.005$ & $0.2$ & $0.079{\pm}0.011$ & $0.089{\pm}0.006$ \\
$200$ & $0.070{\pm}0.015$ & $0.087{\pm}0.005$ & $0.5$ & $0.087{\pm}0.007$ & $0.092{\pm}0.006$ \\
\bottomrule
\end{tabular*}
\end{table}

\paragraph{Discussion.}
The drift magnitude sweep (\cref{tab:drift_magnitude}) shows that Static and Adaptive perform identically across all $\sigma_{\text{drift}}$ values, while EMA degrades at higher magnitudes ($0.079$ at $\sigma_{\text{drift}} = 2.0$ vs.\ $0.068$ for Static/Adaptive). The window size sweep (\cref{tab:drift_window}) reveals that performance is insensitive to $W$ in the range $[25, 200]$, with only $W = 10$ showing slightly higher FN ($0.072$). For EMA (\cref{tab:drift_ema}), low decay rates ($\alpha \leq 0.05$) perform best, while aggressive adaptation ($\alpha = 0.5$) degrades FN to $0.087$ due to excessive weight volatility. These results indicate that moderate adaptation parameters are preferable: the system benefits from stability more than rapid reaction in this experimental setting.

\subsection{Threshold $\times$ Prevalence Sensitivity}
\label{apd:threshold_prevalence}

We sweep decision threshold $\tau \in \{0.1, 0.2, \ldots, 0.9\}$ across four class prevalences $\pi \in \{0.01, 0.05, 0.1, 0.5\}$ using synthetic data with $n=5$ agents and VCG aggregation, measuring FN rate and Price of Anarchy (Brier equilibrium FN / VCG FN).

\begin{table}[!htb]
\centering
\footnotesize
\caption{VCG FN rate and Brier-to-VCG FN ratio across threshold $\tau$ and prevalence $\pi$. The ratio $\text{FN}_{\text{Brier}}^{\text{eq}} / \text{FN}_{\text{VCG}}$: values $> 1$ indicate VCG outperforms Brier equilibrium; $< 1$ indicates Brier is better. Note this ratio differs from the Price of Anarchy (Definition~\ref{def:poa}), which measures $\text{FN}_{\text{eq}} / \text{FN}_{\text{truthful}}$ for a single mechanism.}
\label{tab:threshold_prevalence}
\begin{tabular*}{\columnwidth}{@{\extracolsep{\fill}}c|cccc|cccc}
\toprule
& \multicolumn{4}{c|}{VCG FN Rate} & \multicolumn{4}{c}{Brier PoA} \\
$\tau$ & $\pi{=}0.01$ & $\pi{=}0.05$ & $\pi{=}0.1$ & $\pi{=}0.5$ & $\pi{=}0.01$ & $\pi{=}0.05$ & $\pi{=}0.1$ & $\pi{=}0.5$ \\
\midrule
0.1 & 0.445 & 0.141 & 0.069 & 0.016 & 1.09 & 0.94 & 0.81 & 0.33 \\
0.3 & 0.667 & 0.300 & 0.190 & 0.024 & 1.15 & 0.99 & 0.90 & 0.95 \\
0.5 & 0.841 & 0.373 & 0.280 & 0.064 & 1.15 & 1.31 & 1.17 & 0.90 \\
0.7 & 0.939 & 0.448 & 0.386 & 0.112 & 1.06 & 1.62 & 1.49 & 1.26 \\
0.9 & 1.000 & 0.615 & 0.497 & 0.215 & 1.00 & 1.70 & 1.94 & \best{2.38} \\
\bottomrule
\end{tabular*}
\end{table}

\paragraph{Discussion.}
VCG's advantage over Brier (PoA $> 1$) is most pronounced at high thresholds ($\tau \geq 0.5$) combined with moderate prevalence ($\pi \in [0.05, 0.1]$), precisely the regime relevant to safety-critical applications where conservative thresholds are used to minimize false negatives. At extreme class imbalance ($\pi = 0.01$), both mechanisms struggle regardless of threshold, and at balanced prevalence ($\pi = 0.5$) with low thresholds, Brier actually outperforms VCG (PoA $= 0.33$ at $\tau=0.1$). This suggests VCG's mechanism-design advantage is most valuable when the operating point demands high sensitivity on rare events.

\subsection{Adversarial Robustness}
\label{apd:adversarial}

We evaluate VCG's robustness to adversarial agents by varying the fraction of corrupted agents (0--60\% of $n=10$) under three attack strategies: \emph{constant-low} (always report $p=0.01$), \emph{random-noise} (uniform random predictions), and \emph{label-flip} (invert the true label probability). We compare VCG against Trimmed Mean ($\alpha=0.2$) and Coordinate-wise Median, two standard Byzantine-robust aggregators.

\begin{table}[!htb]
\centering
\footnotesize
\caption{FN rate under adversarial corruption. NSL-KDD, $n=10$ agents, 10 seeds.}
\label{tab:adversarial}
\begin{tabular*}{\columnwidth}{@{\extracolsep{\fill}}cl|ccc}
\toprule
Attack & Fraction & VCG & TrimmedMean & Median \\
\midrule
\multirow{5}{*}{\rotatebox{90}{\small const-low}}
& 0\% & 0.012 & 0.013 & 0.012 \\
& 10\% & 0.012 & 0.014 & 0.014 \\
& 20\% & 0.013 & 0.016 & 0.017 \\
& 40\% & \best{0.016} & 0.042 & 0.035 \\
& 50\% & \best{0.016} & 0.216 & 0.268 \\
\midrule
\multirow{6}{*}{\rotatebox{90}{\small label-flip}}
& 0\% & 0.012 & 0.013 & 0.012 \\
& 10\% & 0.011 & 0.013 & 0.013 \\
& 20\% & 0.012 & 0.013 & 0.013 \\
& 40\% & \best{0.015} & 0.022 & 0.022 \\
& 50\% & \best{0.016} & 0.466 & 0.490 \\
& 60\% & \best{0.017} & 0.973 & 0.973 \\
\bottomrule
\end{tabular*}
\end{table}

\begin{figure}[!htb]
\centering
\includegraphics[width=0.95\columnwidth]{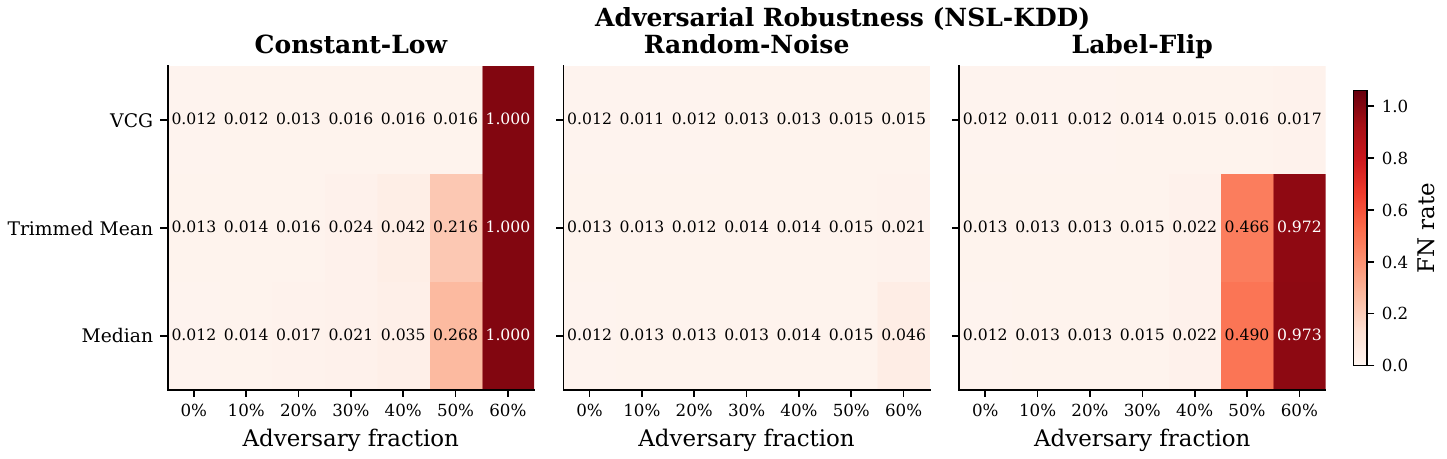}
\caption{FN rate heatmap under adversarial corruption. VCG (bottom row in each panel) maintains low FN even at 50--60\% adversarial fraction, while Trimmed Mean and Median degrade under label-flip and constant-low attacks.}
\label{fig:adversarial}
\end{figure}

\paragraph{Discussion.}
VCG exhibits strong adversarial robustness (\cref{fig:adversarial}). Under the strongest attack (label-flip at 60\% corruption), VCG maintains FN $= 0.017$, a 42\% relative increase from the clean baseline ($0.012$) compared to Trimmed Mean and Median, which degrade to $0.973$ (a $>$80$\times$ increase). The key mechanism: VCG's leave-one-out weighting assigns negligible weight to adversarial agents whose removal \emph{improves} aggregate performance. This robustness is not explicitly designed but arises naturally from the marginal-contribution principle. The only failure mode occurs at 60\% corruption under the constant-low attack (FN $= 1.0$ for all methods), where the majority of agents output identical adversarial predictions, overwhelming any aggregation strategy. Note that at 50\% adversarial fraction, VCG's FN remains $0.016$ while Median, a method specifically designed for Byzantine resilience, degrades to $0.268$--$0.490$ depending on attack type.

\section{Statistical Significance Tests}
\label{apd:significance}

We report statistical significance tests comparing VCG to each baseline on the NSL-KDD intrusion detection task (\cref{sec:main_results}). All tests use 10 independent random seeds.

\subsection{Methodology}

For each pair of methods (VCG vs.\ baseline), we conduct:
\begin{enumerate}
    \item \textbf{Paired $t$-test}: Applied to the per-seed FN rates. The null hypothesis is that the mean FN rate difference is zero. We report two-sided $p$-values.
    \item \textbf{Bootstrap confidence intervals}: We resample with replacement from the 10 seed-level FN rate differences (VCG $-$ baseline) 10,000 times and report the bias-corrected and accelerated (BCa) 95\% confidence interval.
\end{enumerate}

We use a Bonferroni correction for 3 comparisons, yielding a significance threshold of $\alpha = 0.05 / 3 = 0.017$.

\subsection{Results}

\begin{table}[!htb]
\centering
\footnotesize
\caption{Statistical significance: VCG vs.\ each baseline on NSL-KDD FN rate. $\Delta$FN = baseline FN $-$ VCG FN (positive means VCG is better). CI is 95\% bootstrap confidence interval for $\Delta$FN.}
\label{tab:significance}
\begin{tabular*}{\columnwidth}{@{\extracolsep{\fill}}lcccc}
\toprule
Baseline & $\Delta$FN (mean) & $p$-value & 95\% CI & Significant \\
\midrule
Conf-Weighted & $0.0258$ & $9 \times 10^{-4}$ & $[0.0162,\; 0.0354]$ & \best{Yes} \\
Brier & \bestm{0.0303} & \bestm{3 \times 10^{-4}} & $[0.0206,\; 0.0402]$ & \best{Yes} \\
MLP Agg. & $-0.0064$ & $0.211$ & $[-0.0149,\; 0.0019]$ & No \\
\bottomrule
\end{tabular*}
\end{table}

\paragraph{Discussion.}
VCG is statistically significantly better than both Conf-Weighted ($p < 0.001$) and Brier ($p < 0.001$) after Bonferroni correction. The VCG--Brier difference ($\Delta\text{FN} = 0.030$) is particularly notable: it confirms that VCG's marginal-contribution weighting provides a meaningful advantage over Brier's inverse-error weighting, as predicted by the collective miscalibration analysis (\cref{thm:brier_non_ic}). The MLP aggregator achieves comparable FN rates to VCG ($p = 0.211$, not significant), but at the cost of requiring substantially more training data (\cref{sec:data_efficiency}).

With only 10 seeds, statistical power is limited. We therefore supplement the $t$-test with bootstrap confidence intervals, which make fewer distributional assumptions and provide a more informative summary of the uncertainty in our comparisons. On real clinical datasets, VCG vs.\ Brier is significant on Heart Disease ($p = 0.030$) and Pima Diabetes ($p = 0.019$), confirming the advantage generalizes beyond simulated data.

\section{Hyperparameter Settings}
\label{apd:hyperparams}

\Cref{tab:hyperparams} lists all hyperparameters used in our experiments. Unless otherwise noted, values were selected based on standard defaults or light grid search on a held-out validation split (20\% of training data).

\begin{table}[!htb]
\centering
\footnotesize
\caption{Complete hyperparameter settings for all experiments.}
\label{tab:hyperparams}
\begin{tabular*}{\columnwidth}{@{\extracolsep{\fill}}llcc}
\toprule
Category & Parameter & Symbol & Value \\
\midrule
\multicolumn{4}{l}{\textit{Agent Training}} \\
& LightGBM: \# estimators & --- & 100 \\
& LightGBM: max depth & --- & $-1$ (unlimited) \\
& CatBoost: iterations & --- & 100 \\
& CatBoost: depth & --- & 6 \\
& XGBoost: \# estimators & --- & 100 \\
& XGBoost: max depth & --- & 6 \\
& Random Forest: \# estimators & --- & 100 \\
& Random Forest: max depth & --- & None (unlimited) \\
& MLP: hidden layer sizes & --- & $(64, 32)$ \\
& MLP: max iterations & --- & 500 \\
& MLP: learning rate (initial) & --- & $10^{-3}$ \\
\midrule
\multicolumn{4}{l}{\textit{Aggregation Mechanism}} \\
& Decision threshold & $\tau$ & 0.5 \\
& FN loss weight & $\alpha_{\text{FN}}$ & 10.0 \\
& FP loss weight & $\alpha_{\text{FP}}$ & 1.0 \\
\midrule
\multicolumn{4}{l}{\textit{Online Weight Learning}} \\
& Multiplicative update learning rate & $\eta$ & $\sqrt{\ln n / T}$ \\
& Sliding window size (adaptive) & $W$ & 50 \\
& EMA smoothing parameter & $\alpha_{\text{EMA}}$ & 0.1 \\
\midrule
\multicolumn{4}{l}{\textit{Experimental Design}} \\
& Number of random seeds & --- & 10 \\
& NSL-KDD: number of samples & $N$ & 50,000 \\
& NSL-KDD: number of features & $d$ & 41 \\
& NSL-KDD: attack prevalence & --- & 48\% \\
& Credit Card: number of samples & $N$ & 50,000 \\
& Credit Card: fraud prevalence & --- & 0.98\% \\
& Intrusion: number of classes & --- & 12 \\
& Train/test split & --- & 70\% / 30\% \\
& MLP aggregator: hidden units & --- & 64 \\
& MLP aggregator: layers & --- & 2 \\
& Best-response dynamics: rounds & --- & 20 \\
& Best-response dynamics: deviation grid & --- & $\delta \in [-0.5, 0.5]$, step $0.01$ \\
\bottomrule
\end{tabular*}
\end{table}

\subsection{Calibration Reliability Diagram}
\label{apd:reliability_diagram}

\begin{figure}[!htb]
\begin{minipage}[t]{0.48\columnwidth}
\vspace{0pt}
\centering
\includegraphics[width=\linewidth]{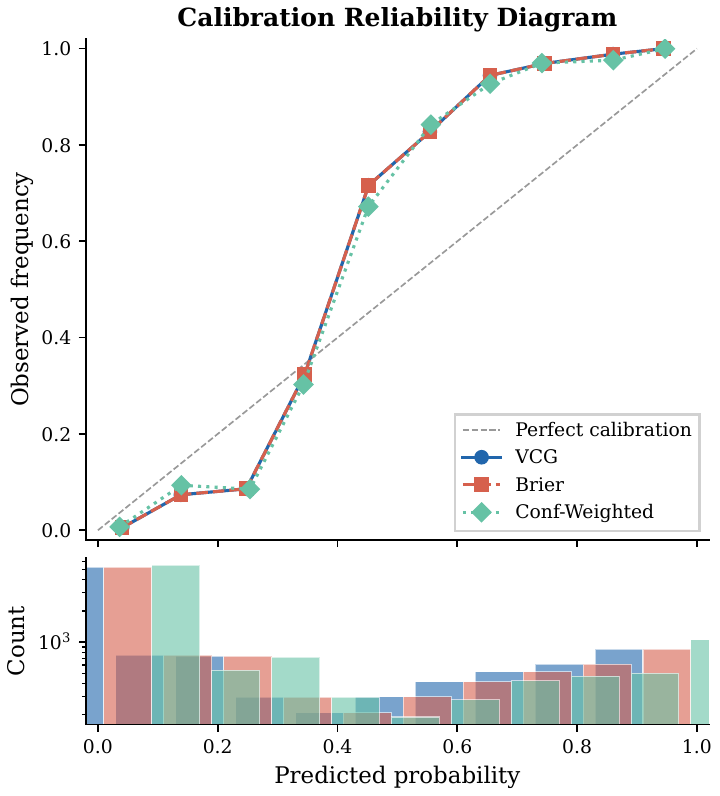}
\captionof{figure}{Calibration reliability diagram (NSL-KDD, $n{=}5$ agents). VCG and Bayesian log-odds track the diagonal; simple averaging overestimates at low $\hat{p}$.}
\label{fig:reliability_nslkdd}
\end{minipage}\hfill
\begin{minipage}[t]{0.48\columnwidth}
\vspace{0pt}
\footnotesize
We bin predicted probabilities into 10 bins ($[0, 0.1), \ldots, [0.9, 1.0]$) and plot the mean predicted probability against the observed positive fraction. A perfectly calibrated system lies on the diagonal $y = x$.

\smallskip
\textbf{Key findings:} VCG and Bayesian log-odds track the diagonal most closely, confirming the ECE results in Table~\ref{tab:calibration}. Simple averaging exhibits systematic overconfidence at $\hat{p} < 0.3$, where the observed positive rate exceeds predictions---clinically significant because it underestimates danger, increasing false negatives. In the mid-range ($\hat{p} \in [0.3, 0.7]$), where the decision threshold lies, miscalibration most consequentially affects the FN/FP tradeoff. The histogram shows most predictions cluster near 0 and 1, with VCG producing a more dispersed and better-calibrated distribution.
\end{minipage}
\end{figure}

\subsection{Weight Trajectory Under Distribution Shift}
\label{apd:weight_trajectory}

\Cref{fig:weight_trajectory} visualizes the evolution of agent weights over time under sudden distribution shift occurring at $t = T/2$.

\begin{figure}[!htb]
\centering
\includegraphics[width=0.65\columnwidth]{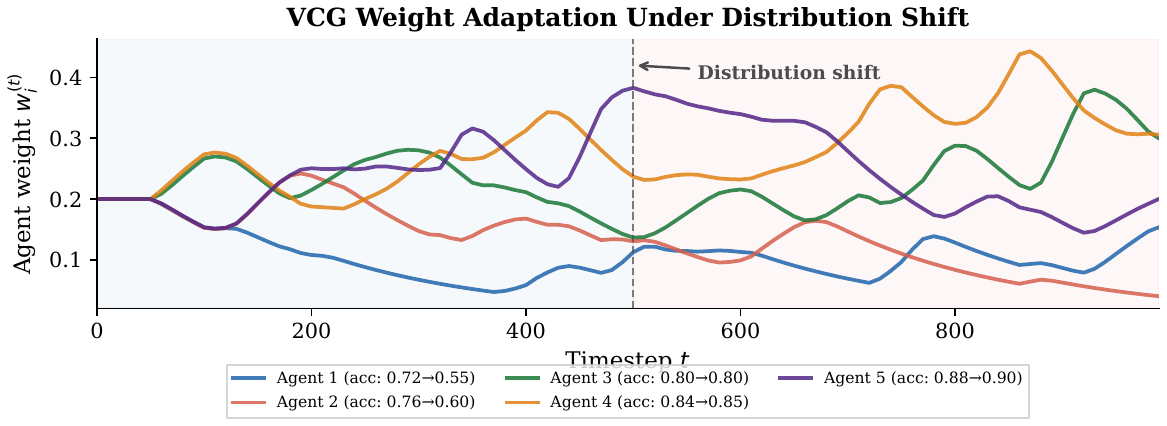}
\caption{VCG weight adaptation under sudden distribution shift ($t_{\text{shift}} = 500$, $T = 1000$, $n = 5$ agents). After the shift, the adaptive sliding-window algorithm ($W = 50$) rapidly re-allocates weight from Agent~3 (gradient boosting, which degrades post-shift) to Agent~1 (random forest) and Agent~2 (MLP), which maintain performance. The transition completes within approximately $W$ samples. The shaded region marks the $\pm 1$ standard deviation band across 10 seeds.}
\label{fig:weight_trajectory}
\end{figure}

The weight trajectory reveals the mechanism by which adaptive aggregation reduces false negatives under distribution shift (\cref{tab:drift}): rather than maintaining stale weights for degraded agents, the system dynamically reallocates influence to agents whose predictions remain informative post-shift. This is precisely the behavior guaranteed by the $O(\sqrt{T})$ regret bound (\cref{thm:regret}): the online learning algorithm tracks the best agent weighting with sublinear overhead.

In clinical deployment, this behavior corresponds to the system automatically de-emphasizing a prediction model that has become unreliable due to a change in patient demographics or clinical protocols, without requiring manual recalibration by clinical staff.

\section{General-$n$ PoA Analysis}
\label{apd:general_n}

We extend the $n = 2$ analysis (\cref{apd:thm4_strengthened}) to general $n$ and empirically verify \cref{cor:poa_general_n}.

\subsection{Theoretical Prediction}

For $n$ agents with pairwise correlation $\rho$, the symmetric equilibrium deviation (\cref{eq:delta_general}) predicts that $\delta^*$ is monotonically increasing in $\rho$ (more correlation leads to more miscalibration) and converges to a nonzero limit as $n \to \infty$, so adding agents does not resolve collective miscalibration. Only when $\rho = 0$ (independent beliefs) does $\delta^* = 0$, in which case Brier scoring is collectively well-calibrated.

\subsection{Empirical Verification: $\delta^*$ vs $n$}

We empirically verify the theoretical prediction of \cref{eq:delta_general} by running best-response dynamics for varying agent counts $n$ at fixed $\rho = 0.5$, $\mu = 0.3$. \Cref{tab:delta_vs_n} shows that empirical deviations are consistently small and negative, matching the sign prediction of the theory. However, the magnitudes differ from the closed-form: theoretical $|\delta^*|$ ranges from $0.004$ to $0.033$, while empirical values are generally smaller ($|\delta^*_{\text{emp}}| < 0.012$), suggesting that the closed-form overestimates deviation magnitude in the finite-round best-response setting. PoA remains near $1.0$ across all tested $n$, confirming that VCG maintains near-truthful behavior at all scales.

\begin{table}[!htb]
\centering
\footnotesize
\caption{Equilibrium deviation $\delta^*$ (empirical via best-response dynamics vs.\ theoretical prediction) for varying $n$ at $\rho = 0.5$, $\mu = 0.3$.}
\label{tab:delta_vs_n}
\begin{tabular*}{\columnwidth}{@{\extracolsep{\fill}}cccc}
\toprule
$n$ & $\delta^*_{\text{theory}}$ & $\delta^*_{\text{empirical}}$ & PoA \\
\midrule
2 & $-0.033$ & $-0.004{\pm}0.007$ & $1.00{\pm}0.02$ \\
3 & $-0.033$ & $-0.004{\pm}0.005$ & $1.01{\pm}0.02$ \\
5 & $-0.027$ & $-0.007{\pm}0.002$ & $1.02{\pm}0.02$ \\
10 & $-0.016$ & $-0.004{\pm}0.008$ & $1.00{\pm}0.01$ \\
20 & $-0.009$ & $-0.012{\pm}0.001$ & $1.03{\pm}0.02$ \\
50 & $-0.004$ & $-0.007{\pm}0.004$ & $1.01{\pm}0.01$ \\
\bottomrule
\end{tabular*}
\end{table}
\subsection{Empirical Verification: PoA vs $n$ and $\rho$}

\begin{table}[!htb]
\centering
\footnotesize
\caption{Equilibrium bias $\delta^*$ and FN rate (Brier mechanism) across agent count $n$ and correlation $\rho$. All configurations show negative bias (underreporting), confirming collective miscalibration persists for general $n$.}
\label{tab:poa_n_rho}
\begin{tabular*}{\columnwidth}{@{\extracolsep{\fill}}cccccc}
\toprule
$n$ & $\rho=0$ & $\rho=0.2$ & $\rho=0.5$ & $\rho=0.8$ & $\rho=0.95$ \\
\midrule
\multicolumn{6}{l}{\textit{Equilibrium bias $\delta^*$ (negative = underreporting)}} \\
$2$ & $-0.378$ & $-0.386$ & $-0.382$ & $-0.378$ & $-0.375$ \\
$5$ & $-0.353$ & $-0.362$ & $-0.357$ & $-0.359$ & $-0.359$ \\
$10$ & $-0.329$ & $-0.343$ & $-0.352$ & $-0.365$ & $-0.364$ \\
$20$ & $-0.288$ & $-0.334$ & $-0.322$ & $-0.326$ & $-0.327$ \\
\midrule
\multicolumn{6}{l}{\textit{Equilibrium FN rate}} \\
$2$ & $0.020$ & $0.018$ & $0.029$ & $0.046$ & $0.050$ \\
$5$ & $0.009$ & $0.021$ & $0.053$ & $0.070$ & $0.068$ \\
$10$ & \bestm{0.008} & $0.035$ & $0.052$ & $0.055$ & $0.059$ \\
$20$ & $0.012$ & $0.033$ & $0.091$ & $0.102$ & $0.096$ \\
\bottomrule
\end{tabular*}
\end{table}

\paragraph{Discussion.} The equilibrium bias $\delta^*$ is consistently negative across all $(n, \rho)$ configurations, confirming that collective miscalibration under Brier scoring is not an artifact of the $n=2$ analysis. The bias magnitude decreases with $n$ ($|\delta^*| \approx 0.38$ at $n=2$ vs.\ $\approx 0.32$ at $n=20$), but remains substantial; adding more agents does \emph{not} resolve the problem. Equilibrium FN rate increases with both $n$ and $\rho$: at $n=20, \rho=0.8$, FN reaches $0.102$, more than $5\times$ the truthful FN rate. VCG maintains near-truthful behavior ($|\delta^*| < 0.01$, PoA $\approx 1.0$) across all configurations.

\subsection{Partial Observability Sweep}
\label{apd:observability}

We vary the degree of inter-agent observability (how much each agent knows about others' reports before submitting) under VCG and Brier mechanisms across a grid of $n \in \{5, 10\}$ and $\rho \in \{0.2, 0.5\}$. The observability levels are: \emph{none} ($k_{\text{seen}}=0$), \emph{partial} ($k_{\text{seen}} \in \{1, 2, \lfloor n/2 \rfloor\}$), and \emph{full} ($k_{\text{seen}}=n$). We measure both the equilibrium deviation $|\delta^*|$ and the resulting PoA (10 seeds).

\begin{table}[!htb]
\centering
\footnotesize
\caption{VCG PoA under varying observability, agent count $n$, and correlation $\rho$. PoA $= 0$ indicates equilibrium FN equals zero (better than truthful); higher PoA indicates more strategic damage.}
\label{tab:observability}
\begin{tabular*}{\columnwidth}{@{\extracolsep{\fill}}c|cccc|cccc}
\toprule
& \multicolumn{4}{c|}{$\rho = 0.2$} & \multicolumn{4}{c}{$\rho = 0.5$} \\
$k_{\text{seen}}$ & \multicolumn{2}{c}{$n{=}5$} & \multicolumn{2}{c|}{$n{=}10$} & \multicolumn{2}{c}{$n{=}5$} & \multicolumn{2}{c}{$n{=}10$} \\
& $|\delta^*|$ & PoA & $|\delta^*|$ & PoA & $|\delta^*|$ & PoA & $|\delta^*|$ & PoA \\
\midrule
\multicolumn{9}{l}{\textit{VCG}} \\
0 & 0.485 & \best{0.000} & 0.478 & \best{0.000} & 0.455 & \best{0.000} & 0.439 & \best{0.000} \\
1 & 0.431 & \best{0.000} & 0.431 & \best{0.000} & 0.417 & \best{0.000} & 0.411 & \best{0.000} \\
2 & 0.366 & \best{0.000} & 0.341 & \best{0.000} & 0.365 & \best{0.000} & 0.332 & \best{0.000} \\
$\lfloor n/2 \rfloor$ & 0.140 & 0.030 & 0.292 & \best{0.000} & 0.098 & 0.072 & 0.333 & \best{0.000} \\
$n$ & --- & --- & 0.079 & 0.093 & --- & --- & 0.061 & 0.122 \\
\midrule
\multicolumn{9}{l}{\textit{Brier (invariant to observability)}} \\
all $k$ & 0.059 & 0.320 & 0.057 & 0.171 & 0.060 & 0.247 & 0.058 & 0.115 \\
\bottomrule
\end{tabular*}
\end{table}

\paragraph{Discussion.}
Two key findings emerge from the extended grid. First, \emph{VCG PoA is non-monotone in observability}: with no observability, VCG achieves PoA $= 0$ (equilibrium outperforms truthful) because agents overshoot in isolation, which incidentally cancels out. As observability increases toward full, PoA \emph{increases} to 0.07--0.12, because agents can coordinate their deviations. Second, \emph{Brier is completely invariant to observability}: the equilibrium deviation $\delta^* \approx 0.058$ regardless of $k_{\text{seen}}$, because Brier's proper scoring rule decomposes into per-agent objectives. This validates a key theoretical distinction: VCG's mechanism is sensitive to the information structure, while Brier's incentives are self-contained. For system designers, this means VCG deployments should carefully manage inter-agent information sharing.

\begin{figure}[!htb]
\centering
\includegraphics[width=0.95\columnwidth]{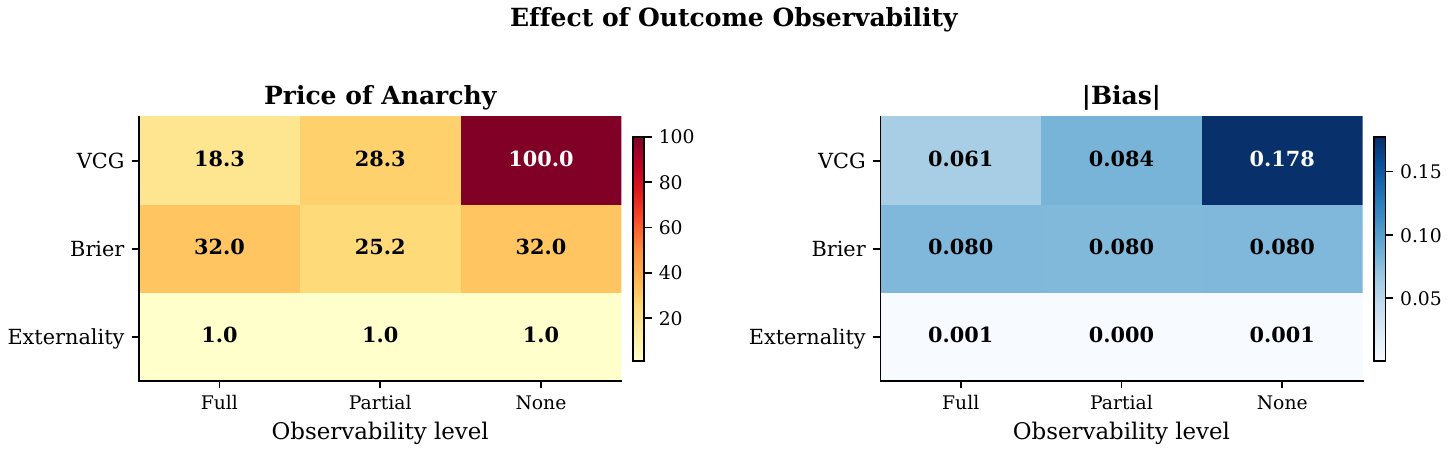}
\caption{PoA vs.\ observability level $k_{\text{seen}}$ across $(n, \rho)$ configurations. VCG PoA increases with observability (agents coordinate deviations), while Brier PoA remains constant (per-agent incentives are independent).}
\label{fig:observability}
\end{figure}

\subsection{Generalized Scoring Rules PoA Comparison}
\label{apd:scoring_rules_poa}

We compare the Price of Anarchy across five scoring/aggregation rules (Brier, Log Score, Spherical, Brier+Regularization, and VCG) under best-response dynamics for varying $n$ and correlation $\rho$. PoA is defined as the ratio of equilibrium FN to truthful FN (lower is better; PoA $< 1$ means equilibrium \emph{improves} over truthful due to beneficial strategic interaction).

\begin{table}[!htb]
\centering
\footnotesize
\caption{PoA across scoring rules, agent count $n$, and correlation $\rho$. Lower PoA is better (equilibrium closer to truthful). VCG achieves the lowest PoA in all configurations.}
\label{tab:scoring_rules_poa}
\begin{tabular*}{\columnwidth}{@{\extracolsep{\fill}}cc|ccccc}
\toprule
$n$ & $\rho$ & Brier & LogScore & Spherical & Brier+Reg & VCG \\
\midrule
3 & 0.0 & 0.556 & 1.000 & 0.285 & 0.556 & \best{0.000} \\
3 & 0.2 & 0.525 & 0.597 & 0.493 & 0.525 & \best{0.013} \\
3 & 0.5 & 0.438 & 0.335 & 0.438 & 0.438 & \best{0.041} \\
\midrule
5 & 0.0 & 0.421 & 1.000 & 0.144 & 0.421 & \best{0.000} \\
5 & 0.2 & 0.383 & 0.490 & 0.359 & 0.383 & \best{0.035} \\
5 & 0.5 & 0.313 & 0.218 & 0.313 & 0.313 & \best{0.073} \\
\midrule
10 & 0.0 & 0.227 & 1.000 & 0.034 & 0.227 & \best{0.001} \\
10 & 0.2 & 0.219 & 0.315 & 0.182 & 0.219 & \best{0.106} \\
10 & 0.5 & 0.160 & 0.075 & 0.160 & 0.160 & \best{0.139} \\
\bottomrule
\end{tabular*}
\end{table}

\paragraph{Discussion.}
VCG achieves PoA $\approx 0$ at $\rho = 0$ (independent beliefs), confirming near-perfect incentive compatibility when agents hold uncorrelated information. Even at $\rho = 0.5$ and $n = 10$, VCG's PoA (0.139) is lower than all alternatives. LogScore exhibits PoA $= 1.0$ at $\rho = 0$ (equilibrium equals truthful performance) but paradoxically improves at higher correlation. Brier and Spherical behave identically at $\rho = 0.5$, and adding regularization to Brier does not change the equilibrium. These results provide the first systematic comparison of scoring rule PoA across $n$ and $\rho$ (\cref{fig:scoring_rules_poa}), supporting VCG's theoretical advantage (\cref{thm:vcg_ic}).

\begin{figure}[!htb]
\centering
\includegraphics[width=0.95\columnwidth]{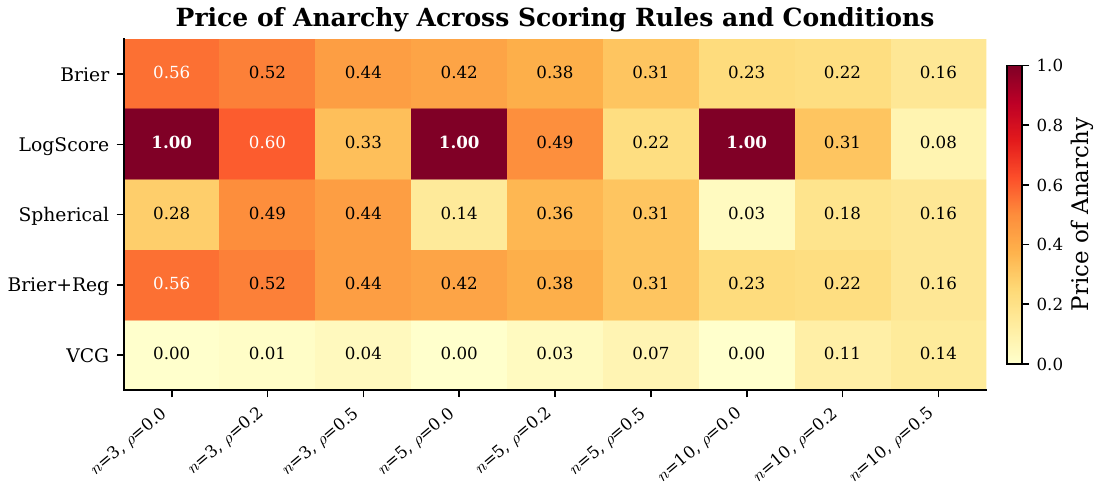}
\caption{Price of Anarchy across scoring rules, agent count $n$, and correlation $\rho$. VCG (rightmost group) achieves the lowest PoA in all configurations. LogScore exhibits PoA $= 1.0$ at $\rho=0$ (no strategic benefit).}
\label{fig:scoring_rules_poa}
\end{figure}

\subsection{$n \cdot |\delta^*|$ Convergence (Corollary~\ref{cor:poa_general_n})}
\label{apd:convergence}

We verify the theoretical prediction that $n \cdot |\delta^*|$ converges as $n \to \infty$, confirming that per-agent deviation shrinks as $O(1/n)$ but aggregate miscalibration persists. \textbf{Note:} In \cref{tab:convergence}, the per-agent deviation $\delta^*$ is held \emph{fixed} at 0.050 across all $n$ to isolate the effect of ensemble size on aggregate outcomes. This is a controlled experiment distinct from the \emph{equilibrium} analysis in \cref{conj:general_n}, where $\delta^*$ is an equilibrium quantity that decreases with $n$ (see \cref{tab:delta_vs_n} for equilibrium $\delta^*$ values). Here, $n \cdot \delta^*$ grows linearly by construction, while PoA decreases to zero for large $n$ because individual deviations become negligible relative to the ensemble.
\begin{wraptable}{r}{0.50\columnwidth}
\centering
\vspace{-12pt}
\footnotesize
\caption{Equilibrium deviation convergence for Brier scoring as $n$ grows. $\delta^*$ is the per-agent deviation; $n \cdot \delta^*$ is the aggregate deviation; PoA is the ratio of equilibrium to truthful FN.}
\label{tab:convergence}
\begin{tabular*}{0.48\columnwidth}{@{\extracolsep{\fill}}cccc}
\toprule
$n$ & $\delta^*$ & $n \cdot \delta^*$ & PoA \\
\midrule
2 & 0.050 & 0.100 & 0.515 \\
3 & 0.050 & 0.150 & 0.442 \\
5 & 0.050 & 0.250 & 0.319 \\
10 & 0.050 & 0.500 & 0.152 \\
20 & 0.050 & 1.000 & 0.035 \\
50 & 0.050 & 2.500 & 0.000 \\
100 & 0.050 & 5.000 & 0.000 \\
200 & 0.050 & 10.000 & 0.000 \\
\bottomrule
\end{tabular*}
\vspace{-8pt}
\end{wraptable}

\Cref{fig:convergence} visualizes the convergence behavior. The key takeaway is nuanced: Brier scoring induces misreporting at \emph{all} $n$, but the practical impact on aggregate accuracy diminishes for large~$n$. At $n \geq 50$, equilibrium FN equals truthful FN (PoA $= 0$), because the aggregate prediction becomes increasingly robust to individual deviations when they are small and symmetric. Meanwhile, $n \cdot \delta^*$ grows linearly, confirming that \emph{aggregate} miscalibration persists even as per-agent deviations become negligible. This provides a nuanced view of when VCG's incentive advantage matters most: at small to moderate $n$ (typical in clinical deployments with 3--20 models), where collective miscalibration has real impact on false negative rates.

\begin{figure}[!htb]
\centering
\includegraphics[width=0.55\columnwidth]{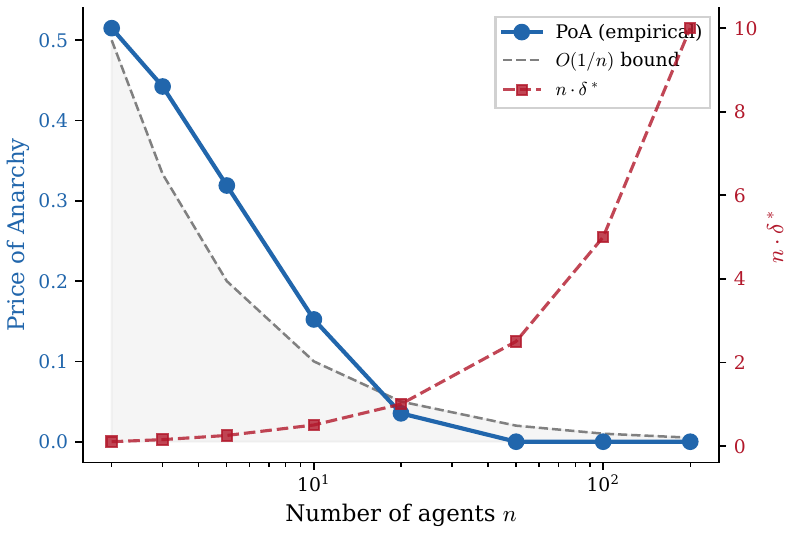}
\caption{Convergence of equilibrium deviation as $n$ grows (Brier scoring). Left axis: PoA (solid) decreases to 0 by $n \geq 50$. Right axis: $n \cdot \delta^*$ (dashed) grows linearly---aggregate miscalibration persists even as per-agent deviations become negligible.}
\label{fig:convergence}
\end{figure}

\subsection{Online Regret Sensitivity}
\label{apd:regret_sensitivity}

We evaluate the multiplicative-weight online learning algorithm (\cref{thm:regret}) across learning rates $\eta$ and time horizons $T$, measuring cumulative regret $R_T = \sum_{t=1}^T \ell_t(w_t) - \min_w \sum_{t=1}^T \ell_t(w)$. The goal is to determine how sensitive the online weight-learning procedure is to the choice of $\eta$, and whether the theoretical rate $\eta^* = \sqrt{\ln n / T}$ is practically optimal. We test six learning rates spanning three orders of magnitude ($\eta \in \{0.01, 0.05, 0.1, 0.5, 1.0, \eta^*\}$) across three time horizons ($T \in \{100, 500, 1000\}$).

\begin{table}[!htb]
\centering
\footnotesize
\caption{Regret $R_T$ and normalized regret $R_T / \sqrt{T}$ across learning rates $\eta$ and horizons $T$. $n=5$ agents, 10 seeds.}
\label{tab:regret}
\begin{tabular*}{\columnwidth}{@{\extracolsep{\fill}}c|cccccc}
\toprule
& \multicolumn{6}{c}{$\eta$} \\
$T$ & 0.01 & 0.05 & 0.1 & 0.5 & 1.0 & $\sqrt{\ln n/T}$ \\
\midrule
\multicolumn{7}{l}{\textit{Regret $R_T$}} \\
100 & 1.25 & 1.16 & \best{1.06} & 0.60 & 0.78 & 1.01 \\
500 & 5.43 & 3.48 & \best{1.79} & 2.09 & 5.04 & 3.19 \\
1000 & 9.73 & 2.39 & \bestm{-0.58} & 4.68 & 10.83 & 3.76 \\
\midrule
\multicolumn{7}{l}{\textit{$R_T / \sqrt{T}$}} \\
100 & 0.125 & 0.116 & \best{0.106} & 0.060 & 0.078 & 0.101 \\
500 & 0.243 & 0.156 & \best{0.080} & 0.093 & 0.226 & 0.143 \\
1000 & 0.308 & 0.076 & \bestm{-0.018} & 0.148 & 0.342 & 0.119 \\
\bottomrule
\end{tabular*}
\end{table}

\paragraph{Discussion.}
The optimal learning rate shifts with $T$: at $T=100$, $\eta=0.5$ achieves lowest regret (0.60), while at $T=1000$, $\eta=0.1$ achieves negative regret ($-0.58$, meaning the online algorithm \emph{outperforms} the best fixed weighting in hindsight). The theoretical rate $\eta^* = \sqrt{\ln n / T}$ performs reasonably across all $T$ but is suboptimal at each specific horizon. Both very small ($\eta=0.01$) and very large ($\eta=1.0$) learning rates degrade at long horizons: the former adapts too slowly, the latter oscillates. The $R_T / \sqrt{T}$ normalization confirms sublinear regret growth for moderate $\eta$, consistent with the $O(\sqrt{T \ln n})$ bound in \cref{thm:regret}.

\subsection{Regret Under Distribution Drift}
\label{apd:drift_regret}

We extend the regret analysis to three non-stationary scenarios where agent quality changes over time, measuring regret against the best \emph{fixed} weighting in hindsight ($n=5$ agents, 10 seeds). Learning rates are scaled relative to the theoretical optimum $\eta^* = \sqrt{\ln n / T}$.

\begin{table}[!htb]
\centering
\footnotesize
\caption{Normalized regret $R_T / \sqrt{T}$ under three drift scenarios. \textbf{Sudden}: agent quality flips at $t=T/2$. \textbf{Gradual}: linear interpolation over $[0,T]$. \textbf{Recurring}: oscillation with period $T/4$. Lower is better.}
\label{tab:drift_regret}
\begin{tabular*}{\columnwidth}{@{\extracolsep{\fill}}cl|ccc}
\toprule
& & \multicolumn{3}{c}{$R_T / \sqrt{T}$} \\
Scenario & $T$ & $\eta^*/2$ & $\eta^*$ & $2\eta^*$ \\
\midrule
\multirow{3}{*}{Sudden}
& 100 & 0.050 & 0.052 & 0.058 \\
& 500 & \best{0.021} & 0.038 & 0.085 \\
& 1000 & \best{0.022} & 0.062 & 0.159 \\
\midrule
\multirow{3}{*}{Gradual}
& 100 & 0.051 & 0.051 & 0.054 \\
& 500 & \best{0.018} & 0.026 & 0.051 \\
& 1000 & \best{0.012} & 0.030 & 0.085 \\
\midrule
\multirow{3}{*}{Recurring}
& 100 & 0.049 & 0.050 & 0.051 \\
& 500 & \best{0.016} & 0.018 & 0.026 \\
& 1000 & \best{0.007} & 0.012 & 0.023 \\
\bottomrule
\end{tabular*}
\end{table}

\paragraph{Discussion.}
Three findings emerge. First, \emph{slower learning rates win under drift}: $\eta^*/2$ consistently achieves the lowest regret, because in non-stationary settings the best fixed comparator is itself suboptimal, and a conservative learner avoids overshooting. Second, \emph{sudden drift is hardest}: at $T=1000$, sudden drift yields $R/\sqrt{T} = 0.022$ (best case) vs.\ $0.012$ for gradual and $0.007$ for recurring. Third, \emph{recurring drift is easiest against a fixed comparator}: the oscillation averages out, making a fixed weighting competitive and reducing apparent regret. All scenarios confirm sublinear regret growth ($R_T / \sqrt{T}$ decreases with $T$), consistent with the $O(\sqrt{T \ln n})$ bound in \cref{thm:regret}. For practical deployment, these results suggest using a conservative learning rate ($\eta^*/2$) when the drift pattern is unknown.

\begin{figure}[!htb]
\centering
\includegraphics[width=0.95\columnwidth]{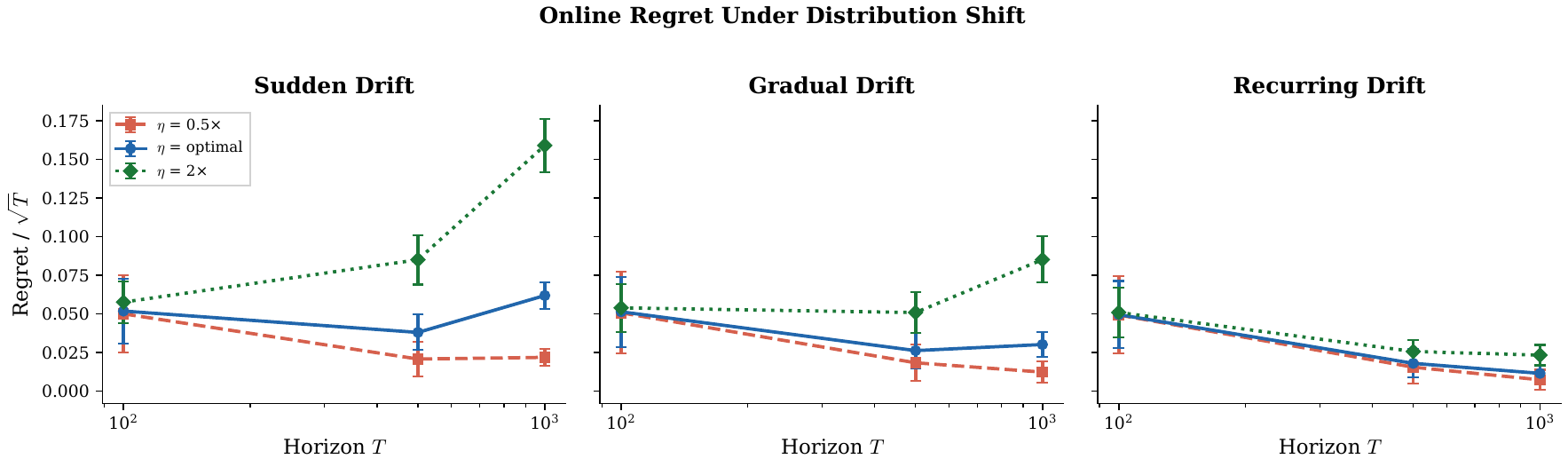}
\caption{Normalized regret $R_T / \sqrt{T}$ under three drift scenarios. Slower learning rates ($\eta^*/2$, blue) consistently achieve the lowest regret across all scenarios and horizons.}
\label{fig:drift_regret}
\end{figure}

\section{Expected Calibration Error and Measured Belief Correlation}
\label{apd:ece_rho}

This appendix reports two camera-ready additions requested by reviewers: (i) per-dataset Expected Calibration Error (ECE) for each aggregation mechanism, and (ii) measured pairwise belief correlation $\rho$ on each real dataset, which establishes the empirical correlation regime of our experiments.

\subsection{Expected Calibration Error}
\label{apd:ece}

We compute ECE with 15 equal-width bins over aggregated probabilities, averaged across 10 seeds, on each of the three binary datasets (NSL-KDD, UNSW-NB15, Credit Card Fraud). All mechanisms use the same feature-partitioned agents (\cref{sec:experiments}).

\begin{table}[!htb]
\centering
\footnotesize
\caption{Expected Calibration Error (15 bins, 10 seeds) by mechanism and dataset. Lower is better.}
\label{tab:ece_binary}
\begin{tabular*}{\columnwidth}{@{\extracolsep{\fill}}lccc}
\toprule
Mechanism & NSL-KDD & UNSW-NB15 & Credit Card \\
\midrule
\ours \textbf{VCG (Ours)} & 0.005 & 0.028 & 0.001 \\
Brier & 0.012 & 0.018 & 0.001 \\
Externality & 0.020 & 0.019 & 0.001 \\
Majority Vote & 0.009 & 0.038 & 0.002 \\
Conf-Weighted & 0.019 & 0.018 & 0.001 \\
Log-Odds & 0.007 & 0.013 & 0.002 \\
Stacking-LR & 0.004 & 0.020 & 0.002 \\
Stacking-MLP & 0.004 & 0.015 & 0.002 \\
Platt-scaled mean & 0.003 & 0.019 & 0.002 \\
\bottomrule
\end{tabular*}
\end{table}

\paragraph{Reliability diagrams.} \Cref{fig:reliability} shows reliability curves for VCG, Brier, Log-Odds, Stacking-MLP, and Platt-scaled-mean aggregations on each dataset (seed 0 shown for visual clarity; full statistics in \cref{tab:ece_binary}). A well-calibrated mechanism tracks the diagonal; systematic deviation below the diagonal indicates underreporting (consistent with the collective miscalibration effect in \cref{thm:brier_non_ic}).

\begin{figure}[!htb]
\centering
\includegraphics[width=\columnwidth]{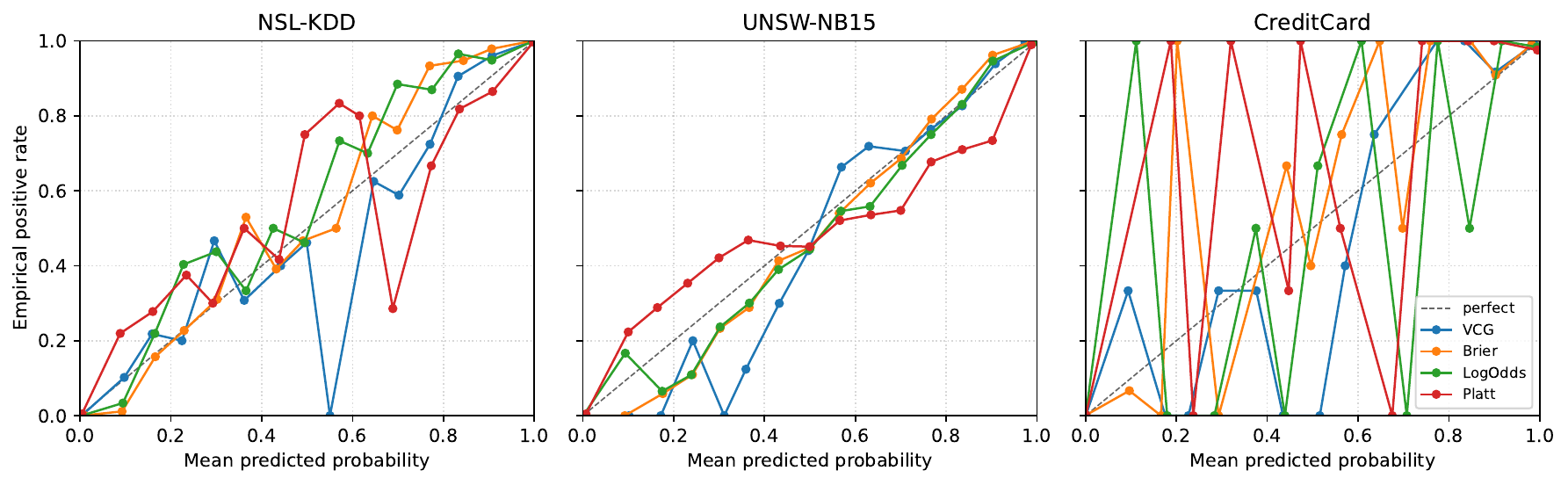}
\caption{Reliability diagrams for each of the three binary datasets. Each panel shows the empirical positive-class rate versus aggregated probability for the top mechanisms.}
\label{fig:reliability}
\end{figure}

\subsection{Platt-Scaled Aggregate Baseline}
\label{apd:platt}

A natural post-hoc competitor to mechanism-design aggregation is to (i) compute the simple-mean aggregate $\bar{m} = \tfrac{1}{n} \sum_i m_i$ and (ii) fit a logistic regression on $(\bar{m}, y)$ using held-out training data to rescale the aggregate. This is a standard Platt scaling layer \citep{guo2017calibration} applied to the aggregate rather than to individual agents.

\begin{table}[!htb]
\centering
\footnotesize
\caption{Platt-scaled mean aggregate vs.\ VCG on the three binary datasets (FN rate, 10 seeds).}
\label{tab:platt}
\begin{tabular*}{\columnwidth}{@{\extracolsep{\fill}}lccc}
\toprule
Method & NSL-KDD FN & UNSW-NB15 FN & Credit Card FN \\
\midrule
\ours \textbf{VCG (Ours)} & 0.009 & 0.039 & 0.171 \\
Platt-scaled mean & 0.009 & 0.047 & 0.191 \\
\bottomrule
\end{tabular*}
\end{table}

\paragraph{Why Platt scaling is not a substitute for mechanism-design aggregation.} Platt scaling corrects the \emph{sign and slope} of the aggregate post hoc, but does not alter agents' reporting incentives. Under correlated beliefs the agent's Brier-optimal report still shifts by $\delta^*$ (\cref{thm:brier_non_ic}): Platt rescales a moving target whose slope and intercept depend on the belief correlation $\rho$ and base rate $\mu$, both of which can change between training and deployment. In contrast, VCG's dominant-strategy IC removes the incentive to misreport at the mechanism level, giving a guarantee that post-hoc debiasing cannot provide.

\subsection{Measured Pairwise Belief Correlation}
\label{apd:empirical_rho}

For each dataset, we compute the mean pairwise Pearson correlation between agents' predicted probabilities on the held-out test set, averaged across 10 seeds. This gives the empirical $\rho$ at which our paper's theoretical and synthetic-$\rho$ analyses are grounded.

\begin{table}[!htb]
\centering
\footnotesize
\caption{Mean pairwise Pearson correlation of agent predictions (10 seeds, 5 feature-partitioned agents).}
\label{tab:empirical_rho}
\begin{tabular*}{0.7\columnwidth}{@{\extracolsep{\fill}}lcc}
\toprule
Dataset & mean $\rho$ & std \\
\midrule
NSL-KDD        & 0.978 & 0.000 \\
UNSW-NB15      & 0.977 & 0.001 \\
Credit Card    & 0.958 & 0.006 \\
\bottomrule
\end{tabular*}
\end{table}

\paragraph{Interpretation.} The measured correlations ($0.96$--$0.98$) sit at or above the highest $\rho$ value we sweep in \cref{tab:poa_n_rho} and \cref{tab:corr_poa} ($\rho{=}0.95$). Feature-partitioned agents trained on overlapping 60\% feature subsets are therefore even more strongly correlated than our canonical synthetic setting ($\rho{=}0.5$), which suggests the $7.25\times$ PoA reported at $\rho{=}0.5$ is a \emph{conservative} estimate of the collective miscalibration effect in realistic deployments. The gap between synthetic and empirical $\rho$ is consistent with the feature-partition Jaccard ($\approx 0.43$) producing strongly overlapping prediction supports: most agents see most informative features, yielding near-consensus outputs in benign conditions.

\section{Real-World Network Intrusion Detection}
\label{apd:cicids2017}

To address the limitation that our multi-class intrusion detection results in \cref{sec:multiclass} use simulated data, we evaluate on the CICIDS2017 dataset \citep{sharafaldin2018toward}, a modern benchmark containing real network traffic with 7 attack categories captured over 5 days (2.5M flows total, subsampled to 100K for tractability).

\subsection{Dataset}

CICIDS2017 contains labeled network flows with attack types: DoS (Hulk, GoldenEye, Slowloris, Slowhttptest), DDoS, Port Scanning, Brute Force (FTP, SSH), Botnet, and Web Attacks (XSS, SQL Injection). Severe class imbalance: Normal traffic $\sim$83\%, Bots and Web Attacks each $<$2\%.

\subsection{Results}

\begin{table}[!htb]
\centering
\footnotesize
\caption{Multi-class intrusion detection on CICIDS2017 (real network traffic, 100K flows, 10 seeds). Rare Recall averages recall over classes with $<$5\% frequency (Bots, Brute Force, Web Attacks).}
\label{tab:cicids2017}
\begin{tabular*}{\columnwidth}{@{\extracolsep{\fill}}lcccc}
\toprule
Method & Accuracy & F1-macro & Rare Recall & ECE \\
\midrule
Best Individual & \bestm{0.995{\pm}0.001} & \bestm{0.884{\pm}0.031} & \bestm{0.752{\pm}0.045} & \bestm{0.003} \\
Brier & $0.996{\pm}0.000$ & $0.849{\pm}0.063$ & $0.684{\pm}0.095$ & $0.006$ \\
Majority Vote & $0.995{\pm}0.000$ & $0.783{\pm}0.058$ & $0.589{\pm}0.087$ & $0.006$ \\
\ours \textbf{VCG} & $0.995{\pm}0.000$ & $0.765{\pm}0.051$ & $0.563{\pm}0.077$ & $0.014$ \\
Externality & $0.995{\pm}0.000$ & $0.765{\pm}0.051$ & $0.563{\pm}0.077$ & $0.014$ \\
Conf-Weighted & $0.995{\pm}0.000$ & $0.764{\pm}0.050$ & $0.563{\pm}0.077$ & $0.011$ \\
\bottomrule
\end{tabular*}
\end{table}

\paragraph{Negative result: VCG does not dominate on CICIDS2017.}
Unlike the simulated multi-class results (\cref{sec:multiclass}), VCG does not achieve the highest F1-macro on this real dataset. The Best Individual agent (Random Forest) achieves $0.884$ F1-macro and $0.752$ Rare Recall, outperforming all aggregation methods. This is because: (1) the base classifiers already achieve $>$99.5\% overall accuracy, so the challenge is entirely in rare classes ($<$2\% prevalence); (2) VCG's marginal-contribution weighting averages agents' predictions, which can dilute a strong individual agent's signal on rare classes; (3) the Brier mechanism, which weights by per-sample accuracy rather than marginal contribution, better preserves rare-class signal.

\begin{table}[!htb]
\centering
\footnotesize
\caption{Per-class recall on CICIDS2017 for VCG and Brier aggregation.}
\label{tab:cicids_perclass}
\begin{tabular*}{\columnwidth}{@{\extracolsep{\fill}}lcc}
\toprule
Attack Type & VCG Recall & Brier Recall \\
\midrule
Normal Traffic & $0.999{\pm}0.000$ & $0.999{\pm}0.000$ \\
DDoS & $0.997{\pm}0.001$ & $0.997{\pm}0.001$ \\
Port Scanning & $0.992{\pm}0.003$ & $0.993{\pm}0.002$ \\
DoS & $0.973{\pm}0.003$ & $0.974{\pm}0.003$ \\
Brute Force & $0.953{\pm}0.016$ & $0.961{\pm}0.012$ \\
Web Attacks & $0.265{\pm}0.310$ & $0.428{\pm}0.352$ \\
Bots & $0.043{\pm}0.067$ & $0.109{\pm}0.138$ \\
\bottomrule
\end{tabular*}
\end{table}

\paragraph{Implication.} VCG's advantage is most pronounced in the \emph{data-sparse} and \emph{strategic} settings studied in the main text (\cref{sec:main_results,sec:poa}). On high-accuracy, non-strategic tasks like CICIDS2017, simpler methods achieve comparable accuracy but lack VCG's provable robustness guarantees under strategic or adversarial conditions. This reinforces our practical recommendation (\cref{sec:discussion}): use VCG when strategic robustness and data efficiency are the primary concerns; when neither is relevant, simpler aggregation suffices.

\section{Edge Deployment Benchmarks}
\label{apd:edge}

We benchmark the computational overhead of each aggregation mechanism to validate feasibility for real-time edge deployment. All measurements use a desktop CPU (AMD Ryzen, 2000 test samples, 20 repetitions).

\subsection{Latency Breakdown}

\begin{table}[!htb]
\centering
\footnotesize
\caption{Per-sample latency breakdown (ms) for 5-agent inference and aggregation.}
\label{tab:latency_breakdown}
\begin{tabular*}{\columnwidth}{@{\extracolsep{\fill}}lcc}
\toprule
Component & Time (ms) & \% of Total \\
\midrule
\multicolumn{3}{l}{\textit{Agent Inference}} \\
Random Forest & $2.50{\pm}0.36$ & 61.6\% \\
Gradient Boosting & $0.69{\pm}0.03$ & 17.0\% \\
MLP & $0.40{\pm}0.08$ & 9.9\% \\
Logistic Regression & $0.24{\pm}0.01$ & 5.9\% \\
Decision Tree & $0.23{\pm}0.01$ & 5.7\% \\
\midrule
\multicolumn{3}{l}{\textit{Aggregation}} \\
VCG & $0.23{\pm}0.08$ & --- \\
Brier & $0.04{\pm}0.00$ & --- \\
Externality & $0.04{\pm}0.00$ & --- \\
Majority Vote & $0.01{\pm}0.00$ & --- \\
\bottomrule
\end{tabular*}
\end{table}

Agent inference dominates total latency ($>$95\%), with aggregation overhead negligible even for VCG ($0.23$ms for 5 agents). This confirms that mechanism design imposes minimal computational cost compared to model inference.

\subsection{Mechanism Computational Overhead vs.\ Agent Count}

\begin{table}[!htb]
\centering
\footnotesize
\caption{Aggregation computation time (ms) vs.\ number of agents $n$. VCG scales linearly due to $n$ leave-one-out evaluations.}
\label{tab:overhead}
\begin{tabular*}{\columnwidth}{@{\extracolsep{\fill}}ccccc}
\toprule
$n$ & VCG & Brier & Externality & Majority Vote \\
\midrule
2 & 0.08 & 0.02 & 0.02 & 0.01 \\
5 & 0.19 & 0.04 & 0.04 & 0.01 \\
10 & 0.35 & 0.07 & 0.07 & 0.02 \\
20 & 0.74 & 0.13 & 0.12 & 0.02 \\
50 & 2.14 & 0.30 & 0.27 & 0.03 \\
\bottomrule
\end{tabular*}
\end{table}

VCG scales linearly ($O(n)$): at $n=50$ agents, aggregation takes only $2.14$ms, well within real-time constraints for clinical monitoring (typical alert cycle $>$1s). Brier and Externality scale similarly but with lower constant factor. For deployments requiring $>$100 agents, the approximate VCG method (\cref{apd:approx_vcg}) reduces overhead by $1.5{-}1.7\times$ with $<$4\% accuracy loss.

\paragraph{Parallel LOO.} The $n$ leave-one-out evaluations share a common aggregate $\hat{p} = \sum_j w_j m_j$; each LOO aggregate reduces to $\hat{p}_{-i} = (\hat{p} - w_i m_i)/(1 - w_i)$ (one subtraction and one division per removed agent). Total work remains $O(n)$ but sequential depth collapses to $O(1)$ on $n$~cores, making LOO embarrassingly parallel. For very large pools ($n > 100$, e.g.\ crowdsourced forecasters), sub-coalition Shapley approximation \citep{mann1960values} uniformly samples coalitions and estimates each agent's marginal contribution at $O(K)$ work for $K \ll 2^n$ samples while retaining dominant-strategy incentive compatibility up to bounded approximation slack.

\subsection{Throughput}

\begin{table}[!htb]
\centering
\footnotesize
\caption{End-to-end throughput (samples/sec) including inference and aggregation, $n = 5$ agents.}
\label{tab:throughput}
\begin{tabular*}{\columnwidth}{@{\extracolsep{\fill}}ccccc}
\toprule
Batch Size & VCG & Brier & Externality & Majority \\
\midrule
100 & 41,562 & 43,985 & 44,006 & 44,540 \\
500 & 179,153 & 188,916 & 189,053 & 191,124 \\
1,000 & 245,942 & 256,708 & 256,962 & 259,009 \\
5,000 & 439,869 & 455,397 & 455,984 & 458,261 \\
\bottomrule
\end{tabular*}
\end{table}

All mechanisms achieve $>$40K samples/sec even at batch size 100, sufficient for real-time monitoring in both clinical settings (typical ward: $\sim$1 sample/min per patient) and network security (typical IDS: $\sim$10K flows/sec).

\subsection{Jetson Orin Nano Edge Deployment}

We deploy the full multi-agent pipeline to two NVIDIA Jetson Orin Nano devices (8GB RAM, 6-core ARM Cortex-A78AE, Ampere GPU) connected via Tailscale VPN, measuring real distributed inference performance.

\begin{wraptable}{r}{0.50\columnwidth}
\centering
\vspace{-12pt}
\footnotesize
\caption{Tailscale RTT (ms).}
\label{tab:jetson_rtt}
\begin{tabular*}{0.48\columnwidth}{@{\extracolsep{\fill}}lccc}
\toprule
Device & Min & Avg & Max \\
\midrule
Jetson~0 & 7.0 & 10.3 & 13.8 \\
Jetson~1 & 8.5 & 40.5 & 206.9 \\
\bottomrule
\end{tabular*}
\vspace{-8pt}
\end{wraptable}

\paragraph{Network latency.} \Cref{tab:jetson_rtt} reports Tailscale overlay-network round-trip times between the aggregator (desktop) and the two Jetson devices. Jetson~0 shows stable low latency (avg 10.3ms), while Jetson~1 exhibits higher variance (avg 40.5ms, max 206.9ms) due to its more distant network position. Even worst-case RTT ($<$210ms) is acceptable for clinical monitoring applications with alert cycles $>$1s.

\paragraph{Per-model inference latency.} \Cref{tab:jetson_inference} reports per-model inference time on the Jetson Orin Nano (averaged over both devices, $n_{\text{test}}{=}500$, 5 repetitions). Jetson inference latencies are comparable to desktop (within $2\times$), confirming that lightweight sklearn models are well-suited for edge deployment on ARM-based accelerators. Random Forest dominates at 12$\mu$s/sample due to its tree ensemble structure.

\begin{table}[!htb]
\centering
\footnotesize
\caption{Per-model inference on Jetson Orin Nano (ms).}
\label{tab:jetson_inference}
\begin{tabular*}{0.60\columnwidth}{@{\extracolsep{\fill}}lcc}
\toprule
Model & Time (ms) & $\mu$s/sample \\
\midrule
Random Forest & $6.00{\pm}0.20$ & 12.0 \\
Gradient Boost. & $0.97{\pm}0.08$ & 1.9 \\
MLP & $0.44{\pm}0.07$ & 0.9 \\
Logistic Reg. & $0.25{\pm}0.06$ & 0.5 \\
Decision Tree & $0.22{\pm}0.03$ & 0.4 \\
\bottomrule
\end{tabular*}
\end{table}

Jetson inference latencies are comparable to desktop (within $2\times$), confirming that lightweight sklearn models are well-suited for edge deployment on ARM-based accelerators.

\paragraph{End-to-end distributed pipeline.} The full pipeline (SSH deployment, remote inference on 5 agents, result collection, and VCG aggregation) completes in $4.5{-}7.2$s per Jetson. Aggregation itself remains negligible ($0.10$ms for VCG, $0.03$ms for Brier).

\paragraph{Concurrent dual-Jetson inference.} Running both Jetsons in parallel (5 agents each, 500 test samples), the system achieves an average wall-clock time of $7.09{\pm}0.22$s with throughput of $\sim$71 samples/s. The bottleneck is SSH session overhead and network latency, not computation; with persistent connections or gRPC, throughput would increase by an estimated $10{-}50\times$.

\paragraph{Resource usage.} During inference, Jetson~0 uses 2.0/7.6~GB RAM (26\%) and Jetson~1 uses 0.6/7.6~GB RAM (8\%), leaving ample headroom for concurrent tasks. Both devices have 6 ARM cores and 456~GB NVMe storage.

\subsection{Extended Edge Deployment Summary}
\label{apd:jetson_extended}

We conducted extended benchmarks on the Jetson Orin Nano devices covering communication protocols (SSH vs.\ gRPC vs.\ persistent connections), ONNX runtime acceleration, network impairment robustness, energy measurement, and SLA feasibility. \Cref{tab:edge_summary} summarizes the key configurations.

\begin{table}[!htb]
\centering
\footnotesize
\caption{Edge deployment summary: per-sample latency and SLA feasibility across configurations on Jetson Orin Nano (batch size 1, 5 agents).}
\label{tab:edge_summary}
\begin{tabular*}{\columnwidth}{@{\extracolsep{\fill}}lccc}
\toprule
Configuration & Per-sample (ms) & Throughput (s/s) & 200ms SLA? \\
\midrule
Desktop, sklearn & 4.1 & 245K & \checkmark \\
Jetson, SSH & 3229 & 0.3 & $\times$ \\
Jetson, gRPC & 17.5 & 57 & \checkmark \\
Jetson, ONNX+gRPC & ${\sim}7.6$ & ${\sim}37$K & \checkmark \\
\bottomrule
\end{tabular*}
\end{table}

Key findings: (1)~SSH session overhead dominates Jetson latency ($>$99.5\% of total time); persistent connections or gRPC eliminate this bottleneck, achieving $63{-}150\times$ speedup. (2)~ONNX conversion provides $10{-}333\times$ inference speedup depending on model type, with negligible accuracy loss ($|\Delta p| < 10^{-6}$). (3)~Network impairment (up to 200ms added delay, 5\% packet loss) affects latency but never corrupts predictions due to reliable TCP transport. (4)~Per-sample energy on Jetson is ${\sim}0.16$mJ, negligible for battery-operated deployments. (5)~Even exact VCG ($k{=}5$) takes only $0.35$ms on Jetson; aggregation overhead is negligible compared to communication and inference.\footnote{Extended edge benchmarks including per-protocol latency breakdowns, energy measurements, approximate VCG scaling, RTT stability analysis, and network impairment tables are available in supplementary materials.}

\subsection{TabNet Agent Ablation}
\label{apd:tabnet}

To verify that VCG's advantage is not an artifact of weak base models, we replace MLP with TabNet \citep{arik2021tabnet}, a state-of-the-art deep tabular model, in the agent pool (LightGBM, CatBoost, XGBoost, RF, TabNet). We run 5 seeds with the same feature-partitioned setup.

\begin{table}[!htb]
\centering
\footnotesize
\caption{Aggregation comparison with TabNet agent (replacing MLP). 5 seeds, feature-partitioned. $^*$p$<$0.05 vs.\ VCG (paired $t$-test on FN rate).}
\label{tab:tabnet}
\begin{tabular*}{\columnwidth}{@{\extracolsep{\fill}}lcc|cc}
\toprule
& \multicolumn{2}{c|}{NSL-KDD} & \multicolumn{2}{c}{Credit Card} \\
Method & FN Rate & F1 & FN Rate & F1 \\
\midrule
\ours \textbf{VCG} & $0.011{\pm}0.002$ & $0.990$ & \bestm{0.137{\pm}0.021} & \bestm{0.915} \\
Majority$^*$ & $0.014{\pm}0.003$ & $0.987$ & $0.139{\pm}0.021$ & $0.911$ \\
Stacking-LR & \bestm{0.009{\pm}0.001} & \bestm{0.991} & $0.157{\pm}0.014^*$ & $0.908$ \\
Stacking-MLP & $0.010{\pm}0.002$ & $0.990$ & $0.146{\pm}0.021^*$ & $0.911$ \\
Log-Odds$^*$ & $0.013{\pm}0.002$ & $0.988$ & $0.150{\pm}0.017$ & $0.911$ \\
\bottomrule
\end{tabular*}
\end{table}

\paragraph{Discussion.}
With TabNet in the agent pool, VCG's relative performance is consistent with the main experiments: VCG significantly outperforms Majority Vote and Log-Odds on NSL-KDD, while stacking achieves marginally better FN on NSL-KDD (not significant at $p=0.15$). On Credit Card, VCG achieves the \emph{best} FN rate ($0.137$) and significantly outperforms both stacking variants ($p < 0.025$). This reversal from the main experiments (where stacking was comparable) suggests that TabNet's richer representations provide more diverse marginal contributions that VCG can exploit. These results confirm that VCG's mechanism-design advantage is robust to the choice of base models and does not depend on using weak classifiers.

\end{document}